\pdfoutput=1
\documentclass[%
reprint,
superscriptaddress,
 amsmath,amssymb,
 aps,
]{revtex4-2}

\usepackage[french, english]{babel}
\newcommand{\selectlanguageorig}{}
\let\selectlanguageorig\selectlanguage
\renewcommand{\selectlanguage}[1]{%
  \def\tmp{#1}\def\tmpfr{fr}%
  \ifx\tmp\tmpfr\selectlanguageorig{french}%
  \else\selectlanguageorig{#1}\fi
}

\usepackage{graphicx}
\usepackage{dcolumn}
\usepackage{bm}
\usepackage{array}
\newcolumntype{P}[1]{>{\centering\arraybackslash}p{#1}}

\usepackage{amsthm}
\newtheorem{theorem}{Theorem}
\newtheorem{lemma}{Lemma}
\newtheorem{corollary}{Corollary}
\newtheorem{definition}{Definition}
\newtheorem{proposition}{Proposition}

\newtheorem{remark}{Remark}

\usepackage[normalem]{ulem}

\usepackage[linesnumbered,ruled,vlined]{algorithm2e}
\usepackage{xcolor}
\usepackage{longtable}
\usepackage{booktabs}
\usepackage{acronym}
\newacro{QP}{quadratic programming}
\newacro{PGD}{projected gradient descent}
\newacro{WFE}{Walsh-Fourier expansion}
\newacro{xWFE}{extended Walsh-Fourier expansion}
\newacro{BDD}{binary decision diagram}
\newacro{xBDD}{extended binary decision diagram}
\newacro{SAT}{Boolean satisfiability}
\newacro{SMT}{Satisfiability modulo theories}
\newacro{CDCL}{conflict-driven clause learning}
\newacro{CLS}{continuous local search}
\newacro{DLS}{discrete local search}
\newacro{COP}{circuit-output probability}
\newacro{LRA}{linear real arithmetic}

\DeclareMathOperator*{\argmin}{arg\,min}

\begin{document}

\preprint{APS/123-QED}

\title{Continuous Optimization for Satisfiability Modulo Theories on Linear Real Arithmetic}

\author{Yunuo Cen}
\email{cenyunuo@u.nus.edu}
\affiliation{Department of Electrical and Computer Engineering,
National University of Singapore, Singapore}

\author{Daniel Ebler}
\email{eblerd@hku.hk}
\affiliation{School of Computing and Data Science,
The University of Hong Kong, Hong Kong, China}

\author{Xuanyao Fong}
\email{kelvin.xy.fong@nus.edu.sg}
\affiliation{Department of Electrical and Computer Engineering,
National University of Singapore, Singapore}

\date{\today}

\begin{abstract}

Efficient solutions for satisfiability modulo theories (SMT) are integral in industrial applications such as hardware verification and design automation. 
Existing approaches are predominantly based on conflict-driven clause learning, which is structurally difficult to parallelize and therefore scales poorly.
In this work, we introduce \textsc{FourierSMT} as a scalable and highly parallelizable continuous-variable optimization framework for SMT. 
We generalize the Walsh-Fourier expansion (WFE), called extended WFE (xWFE), from the Boolean domain to a mixed Boolean-real domain, which allows the use of gradient methods for SMT. 
This addresses the challenge of finding satisfying variable assignments to high-arity constraints by local updates of discrete variables. 
To reduce the evaluation complexity of xWFE, we present the extended binary decision diagram (xBDD) and map the constraints from xWFE to xBDDs. 
We then show that sampling the circuit-output probability (COP) of xBDDs under randomized rounding is equivalent to the expectation value of the xWFEs. 
This allows for efficient computation of the constraints. 
We show that the reduced problem is guaranteed to converge and preserves satisfiability, ensuring the soundness of the solutions. 
The framework is benchmarked for large-scale scheduling and placement problems with up to 10,000 variables and 700,000 constraints, achieving 8-fold speedups compared to state-of-the-art SMT solvers. 
These results pave the way for GPU-based optimization of SMTs with continuous systems.
\end{abstract}

\maketitle

\section*{Introduction}

\ac{SMT} is a class of decision problems that ask whether a variable assignment exists that satisfies an abstract set of constraints. 
The constraints are referred to as first-order logic and combine Boolean logic with, for example, real-valued inequalities~\cite{biere2009handbook}. 
SMTs extend the problem of \ac{SAT} with greater expressiveness~\cite{barrett2018satisfiability}, which is essential when addressing complex scientific and engineering challenges, particularly in operations research~\cite{ansotegui2011satisfiability,bofill2020smt}, formal verification~\cite{lahiri2008back,bjorner2012program}, and artificial intelligence~\cite{ye2023satlm,spallitta2024enhancing}, among others. 
\ac{SMT} problems, such as \ac{SAT}, are generally NP-complete~\cite{cook1971complexity}, making it highly challenging to scale solutions to real-life applications.

Practical SMT solvers are mainly based on \ac{CDCL}, which is an exact method that has first achieved engineering success in solving \ac{SAT} problems ~\cite{een2003extensible,audemard2014lazy,liang2018machine,fleury2020cadical}. 
\ac{CDCL} combines Boolean constraint propagation with non-chronological backtracking, efficiently pruning the search space~\cite{marques2002grasp}. 
To address the more general structure of the constraints in SMT, the solvers employ an extended framework known as CDCL($\mathcal{T}$), which integrates \ac{CDCL} with theory solvers~\cite{barrett2018satisfiability}. 
In practice, this means that the solvers explore possible combinations of truth values for propositional constraints while simultaneously checking that the corresponding numerical constraints are mathematically consistent. Prominent solver instances include \textsc{Z3}~\cite{de2008z3}, \textsc{CVC5}~\cite{barbosa2022cvc5}, \textsc{Yices2}~\cite{dutertre2014yices}, \textsc{OpenSMT2}~\cite{bruttomesso2010opensmt}, \textsc{SMT-RAT}~\cite{corzilius2012smt} and \textsc{MathSAT5}~\cite{cimatti2013mathsat5}. 
However, \ac{CDCL}($\mathcal{T}$)-based solvers often suffer from long runtimes and show poor parallelizability on multi-core architectures. This leads to strongly limited scalability. 
In addition, the performance of the \ac{CDCL}($\mathcal{T}$) methods is highly problem-specific, and local search techniques were shown to sometimes outperform \ac{CDCL}($\mathcal{T}$) on benchmarks such as SMT-LIB with low cutoff times~\cite{li2023local}, as well as concrete applications~\cite{li2024smt}.

Alternative approaches based on \ac{CLS} have been proposed for SAT solving, such as \textsc{FourierSAT}~\cite{kyrillidis2020fouriersat}, \textsc{GradSAT}~\cite{kyrillidis2021continuous}, and \textsc{FastFourierSAT}~\cite{cen2025massively}, which transform Boolean constraints into continuous multilinear polynomials~\cite{o2014analysis} and refine solutions with gradient descent on the objective function. 
This allows for strong parallelization on multi-threading architectures such as GPUs and FPGAs, dramatically reducing computation times and allowing for scaling to large and complex problem instances~\cite{cen2025massively}. 
However, the core theory in \ac{CLS}, spectrum analysis, is inherently defined over Boolean domains and does not extend to real-valued variables. 
Consequently, \ac{CLS} has been incompatible with SMT.

\begin{figure*}[t]
    \centering
    \includegraphics[width=\linewidth]{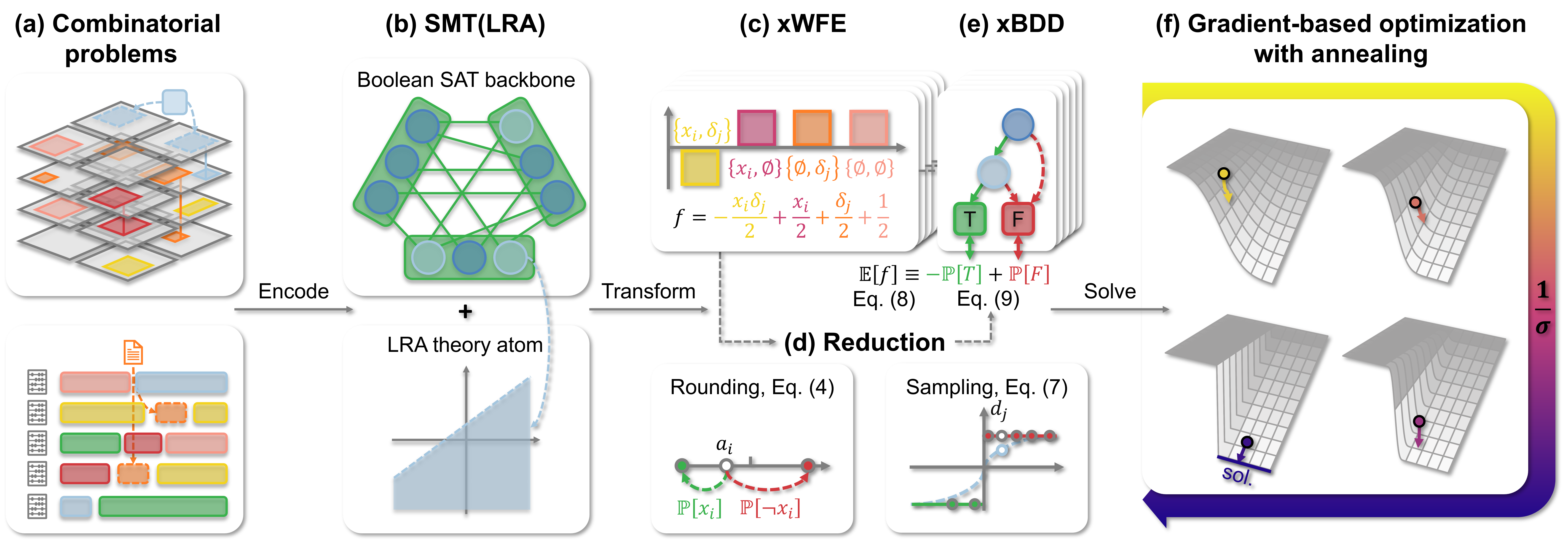}
    \caption{\textbf{The overall flow of \textsc{FourierSMT}.}
        \textbf{(a)} A variety of combinatorial problems target to optimize propositional and real variables.
        \textbf{(b)} These problems can be encoded using an SMT(LRA) formula, which is a SAT formula with additional LRA theory.
        \textbf{(c)} A constraint in the SMT(LRA) formula can be transformed into a polynomial called xWFE. The amplitudes show the coefficients of the xWFE, described in Eq.~\eqref{eq:general-WFE}.
        \textbf{(d)} The expectation of the xWFE can be obtained by randomized rounding on the Boolean variable and the Gaussian sampling on the real variables.  
        \textbf{(e)} However, the xWFE has an exponential number of terms. 
        We show that the circuit output probability of xBDDs is equivalent to the expectation of the xWFE. 
        This avoids state enumeration and thereby preserves computation efficiency. 
        \textbf{(f)} We recover the minima of the original function by annealing of the sampling parameters, gradually restoring the original piecewise objective. 
    }
    \label{fig:flow}
\end{figure*}

In this work, we propose a scalable and highly parallelizable \ac{CLS}-based solver approach for SMT, called \textsc{FourierSMT}, with a focus on \ac{LRA} constraints. 
To enable \ac{CLS} for \ac{SMT}, we introduce a generalization of the \ac{WFE}~\cite{o2014analysis} from the Boolean to the mixed Boolean-real domain to explicitly incorporate \ac{LRA} constraints (see Fig.~\ref{fig:flow}b). 
This allows us to encode the SMT constraints in piecewise multilinear functions, called \ac{xWFE} (see Fig.~\ref{fig:flow}c). 
To identify satisfying variable assignments with gradient-based methods, we construct a smooth surrogate function through \textit{randomized rounding} of Boolean variables and \textit{Gaussian sampling} of real variables (see Fig.~\ref{fig:flow}d). 
The number of terms in an \ac{xWFE} grows exponentially with the number of variables in the corresponding constraint. 
To avoid this overhead, we extend the data structure of \ac{xBDD} to efficiently encode Boolean-continuous constraints (see Fig.~\ref{fig:flow}e). 
We then show equivalence between the expectation value of a \ac{xWFE} and the \ac{COP} of the \ac{xBDD}, which allows one to optimize the \ac{COP} instead. 
This leads to a drastic reduction in complexity. 
The optimality of the solution is shown to be preserved if we optimize the smooth surrogate objective while gradually reducing the smoothness of the original piecewise multilinear function (see Fig.~\ref{fig:flow}f).  
This allows for efficient gradient-based optimization under annealing of the optimization landscape. 
We benchmark \textsc{FouerierSMT} on scheduling and placement problems with up to 10,000 variables and 700,000 constraints and show substantial performance gains of up to 8$\times$ shorter runtimes. 
The results highlight the practical applicability of \textsc{FourierSMT} for complex real-world computational tasks.

\section*{Results}

\subsection*{Satisfiability Modulo Theories} 
\ac{SMT} defines the class of problems that pose the question of whether a set of Boolean constraints is satisfiable with respect to a theory $\mathcal{T}$. 
They were first introduced in the context of extending purely Boolean constraints in SAT to specific mathematical theories~\cite{barrett2002checking,ganzinger2004dpll,nieuwenhuis2006sat}, such as \textsc{Reals}, \textsc{BitVectors}~\cite{jha2009beaver,wright2021modular}, and \textsc{Arrays}~\cite{brummayer2009boolector, ghilardi2010backward}. For \textsc{Reals}, this allows us to combine propositional variables with inequalities on real-valued variables, connected by Boolean connectivities such as ${\land, \lor, \oplus, \to, =}$.

Mathematically, an \ac{SMT} formula $F$ is defined as the conjunction of constraints $c$ from a finite set of Boolean constraints $C$, as
\begin{equation*}
    F = \bigwedge_{c\in{}C}c \ .
\end{equation*}
Solving $F$ requires an assignment of all variables that satisfies all constraints $c\in C$. For example, a typical SMT formula is $F= (x_1\oplus x_2) \wedge (x_1 \vee (y \le 0))$, where $x_1,x_2$ are Boolean and $y\in\mathbb{R}$.

\ac{SMT} solvers are commonly derived from SAT solvers and follow the approach of \ac{CDCL}~\cite{een2003extensible,audemard2014lazy,liang2018machine,fleury2020cadical} and \ac{DLS}~\cite{selman1996generating,selman1994noise,cai2013improving,balint2012choosing,cai2015ccanr}. 
Here, \ac{CDCL} is at its core an exhaustive tree-search on variable assignments, such that for practical scenarios with a large number of constraints compared to the number of decision variables, \ac{CDCL}-based approaches need an infeasibly long time to converge. 
More recently, \ac{CLS} has been explored as a novel complement to \ac{CDCL}, particularly for the problem class of hybrid SAT solving, where Boolean constraints are combined with other types of constraints such as not-all-equal, XOR, or pseudo-Boolean constraints~\cite{kyrillidis2020fouriersat,kyrillidis2021continuous}.
\ac{CLS} has demonstrated orders-of-magnitude speedups over \ac{CDCL} solvers on selected benchmarks~\cite{cen2025massively}.
Existing \ac{CLS} approaches, however, only support Boolean variables and do not extend to SMT. 

In the following, we introduce \textsc{FourierSMT} as a novel \ac{CLS}-based solution framework for SMT, with a focus on the theory of \ac{LRA}. 
In this theory, the constraints assume linear inequalities subject to real numbers. 
SMT(LRA) has a variety of applications, including formal verification~\cite{cordeiro2011verifying}, artificial intelligence~\cite{steiner2010evaluation}, computational biology~\cite{yordanov2013smt}. 
The framework enables continuous optimization techniques to tackle theory-rich constraints that are currently beyond the capability of existing \ac{CLS} methods.

\subsection*{Walsh-Fourier Expansions}
The Walsh-Fourier transform was previously studied as a mathematical tool in theoretical computer science~\cite{bonami1970etude,kahn1988influence}, mainly to analyze the spectral properties of Boolean functions and to prove results in circuit complexity~\cite{tal2017tight}, threshold phenomena in mathematics~\cite{safra2006perspectives}, and social choice theory~\cite{kalai2002fourier}.

Formally, given a function $f:\{\pm{}1\}^n\to\mathbb{R}$ defined on a Boolean hypercube, there exists a unique way of expressing $f$ as a multilinear polynomial~\cite{o2014analysis}, as
\begin{equation}\label{eq:WFE}
    f(x)=\sum_{S\subseteq[n]}\left(\hat{f}(S)\cdot\prod_{i\in{}S}x_i\right) \ .
\end{equation}
Here, $\hat{f}(S)\in\mathbb{R}$ is called the Walsh-Fourier coefficient, the variables $x_i$, $i=1,2,...,n$ are Boolean variables that assume values $\pm1$, and each term in the sum corresponds to a subset of $[n]=\{1,2,...,n\}$.

While the conventional \ac{WFE} applies to purely Boolean domains, SMT problems involve richer structures that include real variables in \ac{LRA}. 
To extend the \ac{WFE} to SMT, we generalize the map in the following to an \ac{xWFE}.

We begin by constructing an atom $\alpha_i$, which in the theory of LRA is a linear equality or inequality of the form
\begin{equation*}
    \alpha_i: \sum_{j=1}^m q_{i,j} y_j - q_{i,0} \bowtie 0,
\end{equation*}
where $\bowtie \in \{=, <, \leq, >, \geq\}$, the coefficients $q_{i,j}$ are numbers that must have a finite syntactic representation so that formulas can be encoded and reasoned about symbolically, and the $y_j$ are real-valued variables. 
We associate with each atomic constraint a Boolean indicator
\begin{equation*}
    \delta_{i}(y) = 
    \begin{cases}
        -1, & \text{if}\; \alpha_i \ \text{is True} \\
        1, & \text{else}.
    \end{cases}
\end{equation*}

Let $c$ be a constraint consisting of $n$ propositional variables and $k$ atoms. 
In SMT(LRA), $c$ can be represented as a function $f_c : \{\pm1\}^n \times \mathbb{R}^m \to \{\pm1\}$.
Then, there is a unique way of expressing $f_c$ as a piecewise multilinear polynomial, where each term corresponds to a pair of subsets of $[n]$ and $[k]$, as (see Supplementary Note~\ref{note:WFE})
\begin{equation}\label{eq:general-WFE}
    f_c\left(x, y\right)=\sum_{\substack{S\subseteq[n] \\ T\subseteq[k]}}\left(\hat{f}_c(S,T)\cdot\prod_{i\in{}S}x_i\cdot\prod_{i\in{}T}\delta_i(y)\right), 
\end{equation}
where $\hat{f}_c(S,T)\in\mathbb{R}$ is the extended Walsh-Fourier coefficient. The expression indicates that each constraint in an SMT(LRA) formula admits a unique piecewise multilinear Walsh-Fourier representation. 
By aggregating the expansions of all constraints, weighted according to their importance, we obtain the global optimization objective of the original SMT(LRA) problem:

\begin{equation}\label{eq:obj}
    \mathcal{F}_w(x,y)=\sum_{c\in{}C}w_c\cdot f_c(x,y) \ .
\end{equation}
Here, weight $w_c \in \mathbb{R}^+$ corresponds to the constraint $c$.

\begin{figure*}[t]
    \centering
    \includegraphics[width=0.8\linewidth]{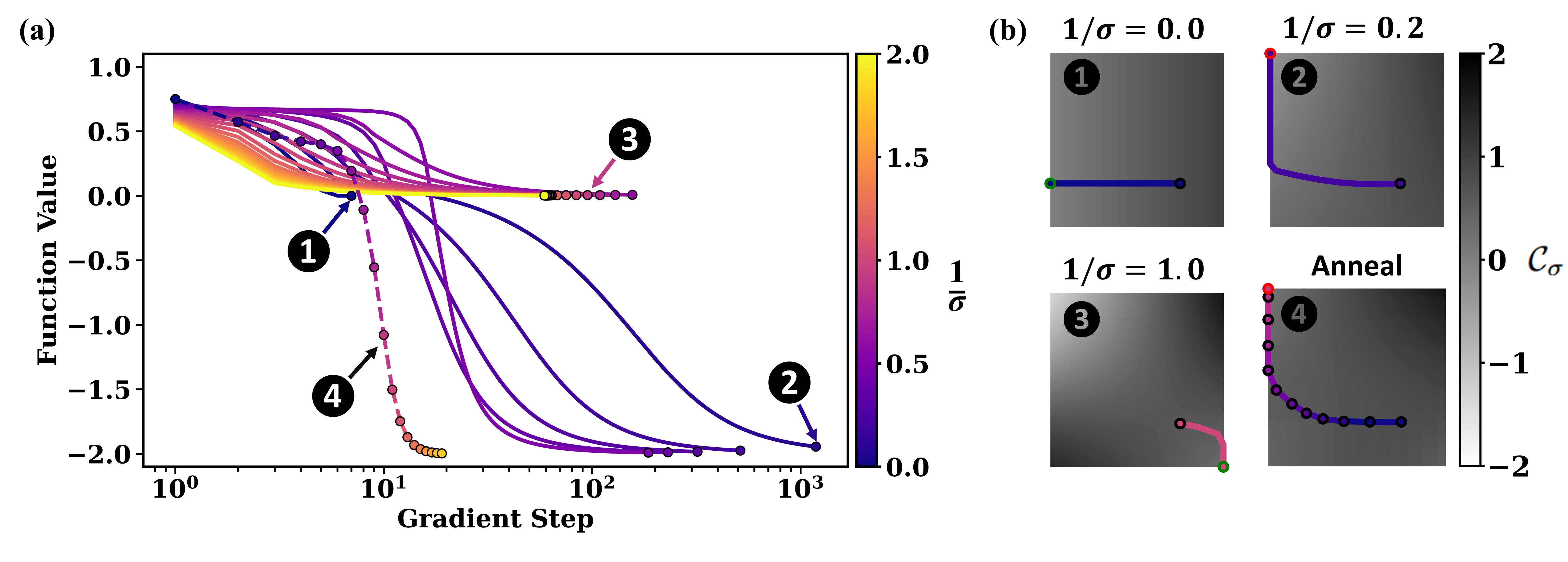}
    \caption{\textbf{Optimization trajectories and energy landscapes under different sampling parameters.}
    The number of gradient steps required to reach an $\epsilon$-projected critical point is shown for $\epsilon=10^{-2}$.
    \textbf{(a)} The colorbar indicates the inverse sampling parameter ${1}/{\sigma}$. 
    \textbf{1.} When $\sigma\to\infty$ (\emph{i.e.}, ${1}/{\sigma}=0$), the gradient with respect to the real-valued variables $y$ vanishes, as the smoothing effect becomes arbitrarily strong.
    \textbf{2.} Moderately large sampling parameters, \emph{i.e.}, ${1}/{\sigma}\in\{0.1,\ldots,0.5\}$, yield slower convergence to the global minimum.
    \textbf{3.} In contrast, smaller values of $\sigma$, corresponding to ${1}/{\sigma}\in\{0.6,\ldots,2.0\}$, lead to faster convergence to a local minimum.
    \textbf{4.} The dashed gradient-colored curve denotes the annealing schedule, in which ${1}/{\sigma}$ is gradually increased in increments of 0.1, achieving both rapid convergence and a final solution close to the global optimum.
    \textbf{(b)} Visualization of the objective landscape and the corresponding optimization trajectory for different $\sigma$. 
    The horizontal and the vertical axes of each subplot correspond to $x$ and $y$ in the domain $[-1,1]$.
    Each subplot shows the smoothed objective ($\mathcal{C}_\sigma$, greyscale background) and the path taken by the gradient descent (colorline, the color refers to the color bar in panel (a) as $1/\sigma$). 
    The optimization trajectories start at the same initial point (black edge circle) and follow the gradients to global (red edge circle) or local minima (green edge circle).}
    \label{fig:trajectory}
\end{figure*}

\subsection*{Rounding and Sampling}
The objective function defined in Eq.~\eqref{eq:obj} is piecewise multilinear, while the argument $x\in \{\pm1 \}^n$ is discrete. 
To make the objective applicable to \ac{CLS}, we relax the propositional variables $x$ to the continuous hypercube $[-1,1]^n$ and optimize the resulting multilinear extension. 

Randomized rounding $\mathcal{R}:[-1,1]^n\to\{\pm1\}^n$ then maps the variables back onto the Boolean hypercube, as

\begin{equation}\label{eq:rounding}
    \begin{cases}
        \mathbb{P}[\mathcal{R}(a)_i=-1] =\frac{1-a_i}{2}\\
            \mathbb{P}[\mathcal{R}(a)_i=+1] =\frac{1+a_i}{2}
    \end{cases}
\end{equation}
where $i\in\{1,\cdots,n\}$ and $a\in[-1,1]^n$. 
The corresponding probability on a Boolean vector is then defined as
\begin{equation*}
    \mathcal{S}_a(x)=\mathbb{P}[\mathcal{R}(a)=x] \ .
\end{equation*}
This provides a probabilistic representation of the \ac{xWFE}. 
Concretely, we express the \ac{xWFE} as its expectation under randomized rounding. 
Let $f_c$ denote the \ac{xWFE} of a constraint $c\in C$. 
Lem.~\ref{lmm:expectation} (see Supplementary Note~\ref{note:expectation}) states that, for any real point $(a, b) \in [-1,1]^n \times \mathbb{R}^m$, the expectation of $f_c$ with respect to the distribution $\mathcal{S}_a$ is
\begin{equation}\label{eq:expectation}
    \mathop{\mathbb{E}}_{x\sim\mathcal{S}_a}[f_c(x,b)] = f_c(a,b), 
\end{equation}
and we have
\begin{equation}\label{eq:expectation-Fw}
    \mathop{\mathbb{E}}_{x\sim\mathcal{S}_a}[\mathcal{F}_w(x,b)] = \mathcal{F}_w(a,b) \ ,
\end{equation}
where we sampled $x$ from the probability distribution $\mathcal{S}_a$.

The following theorem establishes that the global minima of the relaxed problem still encode solutions to the original discrete problem.

\begin{theorem}[Soundness]\label{thm:reduction}
    Let $\mathcal{F}_w$ be the weighted objective function with weights $w_c \in \mathbb{R}^+$ for each constraint $c \in C$.  
    Then the SMT(LRA) formula $F$ is satisfiable if and only if Eq.~\eqref{eq:expectation-Fw} satisfies
    \begin{equation*}
        \min_{\substack{a\in[-1,1]^n \\ b\in\mathbb{R}^m}}\mathcal{F}_w(a,b) = -\sum_{c\in{}C}w_c \ .
    \end{equation*}
\end{theorem}

The proof of Thm.~\ref{thm:reduction} is in the Supplementary Notes~\ref{note:proof}. 

The function in Eq.~\eqref{eq:expectation-Fw} is smooth in $a\in[-1,1]^n$, but non-smooth in $b\in\mathbb{R}^m$, so that the gradient with respect to $b$ vanishes. 
To address this, we introduce a sampling-based smoothing technique that provides a continuous approximation of the atoms.
For $b\in\mathbb{R}^m$, a sampling point $d_i(b)$ drawn from the Gaussian distribution $\mathcal{N}(b,\sigma^2I)$ can be wrtiten as

\begin{align}
    \nonumber d_i(b)=&\mathop{\mathbb{E}}_{y\sim\mathcal{N}(b,\sigma^2I)}\delta_i(y) \\
    &=\mathrm{erf}\left(
        \frac{\sum_{j=1}^m q_{i,j}b_j-q_{i,0}}{\sqrt{2\sum_{j=1}^m q_{i,j}^2}\sigma}\label{eq:sampling}
    \right)
\end{align}
where the second equation is the analytical expression derived when $\delta_i(y)$ corresponds to a linear constraint of the form $q_i^\top y \leq q_{i,0}$, (see the Supplementary Note~\ref{note:expectation}, proof of Lem.~\ref{lmm:expectation}).

Thm.~\ref{thm:expectation} in the Supplementary Note~\ref{note:expectation} states that, for any point $(a,b) \in [-1,1]^n \times \mathbb{R}^m$, we have
\begin{equation}\label{eq:expectation2}
   \mathop{\mathbb{E}}_{\substack{x\sim\mathcal{S}_a \\ y\sim\mathcal{N}(b,\sigma^2I)}} [f_c(x, y)] =\sum_{\substack{S\subseteq[n] \\ T\subseteq[k]}}\left(\hat{f}_c(S,T)\cdot\prod_{i\in{}S}a_i\cdot\prod_{i\in{}T}d_i(b)\right),
\end{equation}
which is the expectation of Eq.~\eqref{eq:general-WFE} under relaxation of variables $x_i$ and $\delta_i$ to $a_i$ and $d_i$. 

The number of terms in Eq.~\eqref{eq:expectation2} grows exponentially as $\mathcal{O}(2^{n+k})$. 
This is because \ac{xWFE} enumerates all possible states. In the following, we present an equivalent approach that significantly reduces the complexity. 
The Cor.~\ref{cor:cop=wfe} in the Supplementary Note~\ref{note:WMI} establishes that the expectation of an \ac{xWFE} can be expressed as a \ac{COP} of a probablistic circuit,
\begin{equation}
    \text{COP}_c = \mathop{\mathbb{E}}_{\substack{x\sim\mathcal{S}_a \\ y\sim\mathcal{N}(b,\sigma^2I)}} \left[f_c(x, y)\right].
\end{equation}
Further, let $\mathcal{C}_\sigma(a,b)=\sum_{c\in{}C}w_c\cdot\text{COP}_c$. Then, we have
\begin{equation}\label{eq:objective}
    \mathcal{C}_\sigma(a,b) = 
    \mathop{\mathbb{E}}_{\substack{x\sim\mathcal{S}_a \\ y\sim\mathcal{N}(b,\sigma^2I)}} \left[{\mathcal{F}_w(x, y)}\right].
\end{equation}

We show in the Supplementary Note~\ref{note:MP} that running the probability assignment algorithm on an \ac{xBDD} can produce such a COP (see~\cite{thornton1994efficient, kyrillidis2021continuous} for an introduction to BDDs, and the Supplementary Note~\ref{note:DD} for our extension to xBDD). 
This allows us to use the \ac{COP} of \ac{xBDD}s as an optimization surrogate instead of the expectation of \ac{xWFE}. 
In addition, the \ac{COP} of an \ac{xBDD} can be efficiently evaluated \cite{bryant1995binary} (see also the Supplementary Note~\ref{note:WMI}), avoiding the problem of state enumeration. 
In the Supplementary Note~\ref{note:MP}, Thm.~\ref{thm:forward} and Thm.~\ref{thm:backward} established that the computational complexity of evaluating the COP scales is $\mathcal{O}\left(|V_c|\right)$, while the complexity of computing its gradient scales is $\mathcal{O}\left(m \cdot |V_c|\right)$, where $|V_c|$ denotes the number of nodes in the decision diagram representing constraint $c$. 
Especially when the literals are symmetric, the complexity of evaluating the \ac{COP} reduces to $\mathcal{O}\left((n+k)^2\right)$~\cite{sasao1996representations}. 
This offers substantial reductions in computational cost~\cite{kolb2018efficient,dos2019exact,kolb2020exploit}.

\subsection*{\textsc{FourierSMT} - A Gradient-Based Optimizer for SMT}

The theoretical framework established above sets the foundation for \textsc{FourierSMT}, a \ac{CLS}-based optimizer for SMT. 
We refer to Alg.~\ref{alg:cls} for the following discussion, and to Supplementary Notes~\ref{note:WFE}-\ref{note:MP} for a detailed analysis of the algorithm. 

\begin{algorithm}[t]
\caption{\textsc{FourierSMT}}
\label{alg:cls}
\KwIn{A constraint set $C$, constraint weights $\{w_c\}$, initial point $(a, b)$, sampling parameter $\sigma$}
\KwOut{Result \textsc{SAT} or \textsc{Unknown}}
\For{$c\in{}C$}{
    Transform $c$ into xWFE $f_c$ and $\text{COP}_c$ \tcp*[r]{Corollary~\ref{cor:wfe} \& Supplementary Notes~\ref{note:WFE}-\ref{note:WMI}}
}
\While{$||g(a,b)||^2 > \epsilon^2$}{
    \For{$i \in \texttt{range}(n)$}{
        $p_i\gets\frac{1-a_i}{2}$ \tcp*[r]{Rounding, Eq.~(\ref{eq:rounding})}
    }
    \For{$i \in \texttt{range}(m)$}{
        $p_{i+n}\gets\frac{1-d_i(b)}{2}$ \tcp*[r]{Sampling, Eq.~(\ref{eq:sampling})}
    }
    Get the projected gradient $g(a,b)$ by Eq.~\eqref{eq:grad_map}\;
    Update the assignment $(a,b)\gets(a,b)-\eta\cdot g(a,b)$\;
}
Rounded assignment $(x, y) \gets (\texttt{sgn}(a), b)$\;
\If{$\mathcal{F}_w (x, y) = -\sum_cw_c$ }{
    \Return \textsc{SAT} \tcp*[r]{Theorem~\ref{thm:reduction}}
}\Else{
\Return \textsc{Unknown} \tcp*[r]{CLS is incomplete}
}
\end{algorithm}

Alg.~\ref{alg:cls} starts by transforming the mixed Boolean-real constraints into continuous functions (lines~1 and 2). The corresponding xBDD can be efficiently constructed using existing decision diagram packages such as CUDD~\cite{somenzi1997cudd}. This enables a gradient-based approach to solve the continuous objective in Eq.~\eqref{eq:objective}, where the minima encode the solution to the original discrete problem. 
The variables are transformed to probabilities due to rounding and sampling (lines~4-7). 
Then, \textsc{FourierSMT} runs \ac{PGD} until it converges to an $\epsilon$-projected-critical point (see Methods).   
At each gradient step, \textsc{FourierSMT} computes the gradient of the \ac{COP} of each constraint $c$ (line 8).  
The gradients can be computed efficiently by top-down and bottom-up message propagations through \ac{xBDD}s (see the Supplementary Note~\ref{note:MP}, Alg.~\ref{alg:forward} and Alg.~\ref{alg:backward}). 
The projected gradient mapping is then computed by Eq.~\eqref{eq:grad_map}, through convex \ac{QP} (see Methods, Prop.~\ref{prop:qp}). 
Line~10 executes the randomized rounding on the propositional variables $x$, mapping the assignment back to the original mixed Boolean-real domain.  
Given the hardness of optimizing non-convex functions~\cite{jain2017non}, it is more practical to converge to local optima and check with Theorem~\ref{thm:reduction} if any optimum is global (line~12).

The performance of \ac{CLS}, measured by the number of steps until convergence, is critically dependent on the smoothness of the objective. 
We show in the Supplementary Note~\ref{note:smoothness}, Thm.~\ref{thm:max_step} that the \ac{PGD} in \textsc{FourierSMT} converges to an $\epsilon$-projected-critical point in $\mathcal{O}(\frac{\alpha L}{\epsilon^2})$ iterations, where $\alpha$ is the sum of the constraint weights and $L$ is the gradient Lipschitz constant.
In the Supplementary Note~\ref{note:smoothness}, Prop.~\ref{prop:lipschitz}, we showed that $L$ depends on the number of variables, constraints, theory atoms, and the sampling parameter $\sigma$.
These parameters jointly determine the convergence rate of each \ac{PGD} trail.

We examine the empirical behavior of the projected gradient method with a concrete example, see Fig.~\ref{fig:trajectory}. 
Here, we study the conjunction of the two constraints $\neg x \oplus (y>0)$ and $x \wedge (y>0)$, which has a local minimum on $x=1$,  $\delta(y)=1$ and a global minimum on $x=-1$,  $\delta(y)=-1$.  
A larger variance $\sigma$ in the Gaussian sampling distribution $\mathcal{N}(b,\sigma^2 I)$ generally produces a smoother and better-conditioned objective landscape. For $\sigma>2$, the Gaussian smoothing covers a wider portion of the energy landscape, enabling the optimizer to capture global structural information and eventually reach the global minimum. 
However, we observe that trials with larger $\sigma$ exhibit slower convergence toward the $\epsilon$-projected-critical point, because excessive smoothing reduces the effective update step size. 
In contrast, when $\sigma < 2$, the sampling becomes more localized, the energy landscape appears sharper, and the optimization rapidly converges to a nearby local minimum.

To mitigate this issue, we adopt an \emph{annealing strategy}, in which optimization begins with a large variance $\sigma$ and gradually reduces it according to a schedule. 
An effective annealing schedule should allow iterations to cross transition regions (see the Supplementary Note~\ref{note:local_smoothness}, Cor.~\ref{cor:lipschitz}) quickly, while subsequently reducing $\sigma$ to recover the exact objective. 
In the limit of $\sigma \to 0$, we recover the exact multilinear objective as in Eq.~\eqref{eq:obj}.
This leads to a global exploration of the landscape and identification of the right neighborhood of solutions. 
As $\sigma$ reduces, the transition regions become tighter, effectively narrowing the landscape down to the neighborhood of a minimum. 
This is formalized by the following

\begin{theorem}[Optimality]\label{thm:optimality}
    When $\sigma = 0$, Eq.~\eqref{eq:objective} has no local optima within the open domain $(-1,1)^n \times \mathbb{R}^m$, \emph{i.e.}, projected gradient will converge to the boundary of $[-1,1]^n \times \mathbb{R}^m$.
\end{theorem}

The proof of Thm.~\ref{thm:optimality} is in the Supplementary Note~\ref{note:optimality}.

Thm.~\ref{thm:optimality} states that reduction of the sampling parameter $\sigma$ to zero recovers the soundness property of Thm.~\ref{thm:reduction}, as it ensures that the smoothing mechanism introduced for optimization does not alter the semantics of the underlying SMT(LRA) formulation.
At the same time, the theoretical guarantees at $\sigma=0$ do not fully characterize the behavior of the optimization dynamics at finite $\sigma$. 
In practice, the overall performance of the solver depends on several factors, including the optimization method, initialization, and, in particular, the annealing schedule used to decrease $\sigma$. 
As in other annealing-based optimization frameworks, the schedule influences how effectively the algorithm explores the landscape before localizing the search to specific regions, and thus plays an important role in determining empirical performance~\cite{triki2005theoretical,morita2008mathematical,bohm2018understanding}. 
In Fig.~\ref{fig:trajectory}b, we illustrate a trial in which $\sigma$ decreases at each gradient step. 
In step 19, the point converges to a $\epsilon$-projected critical point.

The above analysis confirms both the convergence and soundness of the proposed annealing-based optimization scheme, ensuring that the obtained solutions are theoretically grounded.
We next validate these properties empirically through extensive experiments on random hybrid benchmarks and large-scale combinatorial problems.

\begin{figure*}[t]
    \centering
    \includegraphics[width=\linewidth]{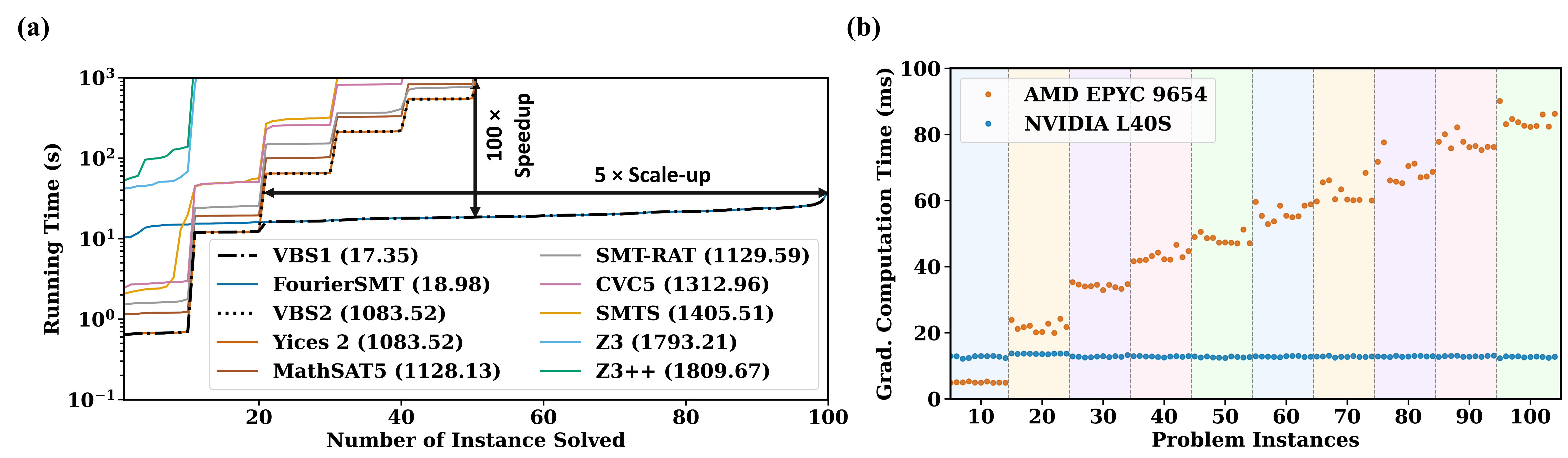}
    \caption{\textbf{Benchmark results on random instances with mixed Boolean-real hybrid constraints.}
    \textbf{(a)} The results show the number of solved instances and runtime efficiency (average PAR-2 score of the solvers). It can be seen that the virtual-best solver (\textsc{VBS1}, including \textsc{FourierSMT}) is 62.45$\times$ faster and solves 50 more instances than its counterpart (\textsc{VBS2}, excluding \textsc{FourierSMT}). 
    The 20 easy instances (small-scale) are solved more efficiently by CDCL($\mathcal{T}$) solvers, whereas all larger instances are solved most efficiently by \textsc{FourierSMT}. This pattern demonstrates that \textsc{FourierSMT} is substantially more scalable.  \textbf{(b)} The average gradient computation time of \textsc{FourierSMT} using CPU (AMD EPYC 9654) and GPU (NVIDIA L40S). 
    For each problem instance, we measure the time required to compute a single gradient step.
    The GPU exhibits near-constant runtime across all sizes, reflecting efficient parallel amortization of fixed overheads, \emph{e.g.}, kernel launch, memory movement. 
    In contrast, CPU runtime increases steadily with instance size, scaling approximately linearly.
    }
    \label{fig:random}
\end{figure*}


\subsection*{Performance Analysis and Application to Combinatorial Problems}
To evaluate the performance of \textsc{FourierSMT}, we first conduct experiments on a benchmark set of random problem instances, each consisting of hybrid constraints following the setting in~\cite{kyrillidis2020fouriersat}, \emph{i.e.}, cardinality ($card$), not-all-equal ($nae$), and parity ($xor$) constraints. 
Such constraints are versatile for encoding a wide range of combinatorial problems.
We vary $n \in \{100, 200, \ldots, 1000\}$ to generate 10 benchmark instances for each $n$.
Each instance contains $n$ Boolean variables and $n$ real variables (this choice is arbitrary, and our solver works for any proportion of Boolean and real variables).
We set the number of each type of constraint to $m_{\text{card}} = m_{\text{nae}} = n/5$, and $m_{\text{xor}}=n/50$, while their respective constraint lengths are $l_{\text{card}} = l_{\text{nae}} = \min(50, n/5)$ and $l_{\text{xor}} = 50$. 
The behavior and interactions of such hybrid Boolean constraints have been studied extensively in prior work, and transitions in the hardness of the problems were studied in~\cite{dudek2016combining,pote2019phase,gupta2020phase}.
We extend the setting here to Boolean and real variables. 
In addition, each instance includes $n$ LRA atoms, enabling the extension of the benchmark to continuous variables.
The details of the random benchmark are given in the Methods.

We compare \textsc{FourierSMT} with representative SMT solvers, including \textsc{Z3}~\cite{de2008z3}, \textsc{Z3++}~\cite{li2023local}, \textsc{CVC5}~\cite{barbosa2022cvc5}, \textsc{Yices2}~\cite{dutertre2014yices}, \textsc{SMTS}~\cite{bruttomesso2010opensmt}, \textsc{MathSAT5}~\cite{cimatti2013mathsat5}, and \textsc{SMT-RAT}~\cite{corzilius2012smt}.
All solvers are applied to the corresponding SMT(LRA) encodings, where the Boolean backbone is expressed in conjunctive normal form.
For reference, we also report two virtual best solvers: \textsc{VBS1}, defined as the best performance achievable across all solvers, and \textsc{VBS2}, defined as the best performance excluding \textsc{FourierSMT}.
The gap between VBS1 and VBS2 quantifies the advantage of \textsc{FourierSMT} compared to the current state-of-the-art.  
The solver performance is assessed using the metric called PAR-2 (penalized average runtime 2) score, as in Eq.~\eqref{eq:PAR2}. 
This is the most common metric used in SMT and SAT solving \cite{amadini2023evaluation}. For all problems, we set a $1000$-second time limit for the benchmark problems. 
The solvers are executed on AMD EPYC 9654 CPUs, while \textsc{FourierSMT} is run on NVIDIA L40S GPUs, see Methods for more details.

For CDCL($\mathcal{T}$)-based solvers, the encoding of instances generated with $n \le 500$ in CNF form creates an overhead such that the encoded formulas contain up to 64,000 variables and 127,000 clauses (see Methods section). 
For \textsc{FourierSMT}, only 1000 variables ($n$ Boolean variables and $n$ real variables) are needed. As a consequence, all solvers except FourierSMT fail to solve instances for which $n \ge 500$. 
\textsc{FourierSMT} scales to problems with up to $n = 1000$. 
This would involve more than 130,000 variables and 250,000 clauses for CDCL($\mathcal{T}$)-based solvers, rendering problems of this magnitude extremely challenging for current state-of-the-art SMT solvers. 

\begin{figure*}[t]
    \centering
    \includegraphics[width=\linewidth]{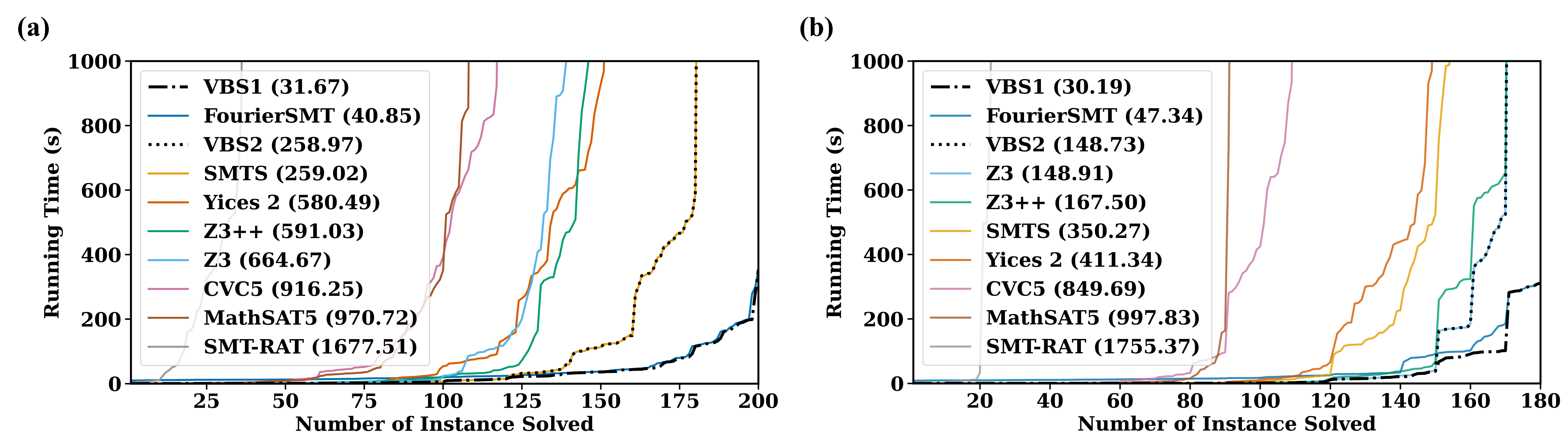}
    \caption{\textbf{Performance comparison of SMT solvers on combinatorial problems.}
        The results show the number of solved instances and runtime efficiency (average PAR-2 score of the solvers).
        \textbf{(a)} Scheduling problems, which mainly consist of non-overlap, feasibility, and task-dependency constraints. The virtual-best solver (\textsc{VBS1}, including \textsc{FourierSMT}) is 8.18$\times$ faster and solves 20 more instances than the counterpart (\textsc{VBS2}, excluding \textsc{FourierSMT}).
        \textbf{(b)} Placement problems, which mainly consist of non-overlap, feasibility, and routing-aware constraints. 
        In this category, \textsc{Z3} and \textsc{Z3++} achieve the best performance cross CDCL($\mathcal{T}$) solvers. \textsc{VBS1} is 4.93$\times$ faster and solves 10 more instances than \textsc{VBS2}. All solvers operate on identical encodings, indicating that the performance gain stems purely from the optimization framework rather than from encoding-free advantages.}
    \label{fig:overall}
\end{figure*}

Despite this hardness, \textsc{FourierSMT} successfully solves all 100 instances. 
In contrast, the portfolio of all competing solvers, \emph{i.e.}, VBS2, solves only 50 small- to medium-scale instances. 
This demonstrates the superior scalability and robustness of FourierSMT on large random hybrid benchmarks. 
From Fig.~\ref{fig:random}a we can observe that, although \textsc{FourierSMT} is relatively slower on small-scale instances ($n=100$), it achieves  100$\times$ speedup on medium-scale instances ($n=500$) and is capable of solving instances that are at least twice as large as those handled by competing solvers within this benchmark.
Within the maximum runtime of \textsc{VBS1} across all instances (37.84 seconds), \textsc{VBS2} only solves 20 instances, reflecting a five-fold improvement in scalability contributed by \textsc{FourierSMT} (see per-instance results in the Supplementary Note~\ref{note:detail}, Table~\ref{tab:random}).

To determine the extent of GPU-induced speedup, we conducted an additional experiment that compares the gradient evaluation time of FourierSMT on CPU and on GPU.
For each instance, we measure the average gradient computation time on a total of 10,000 random points. 
As shown in Fig.~\ref{fig:random}b, for the smallest instance category, the CPU is 2.55$\times$ faster than the GPU.
This is expected, as the workload is insufficient to saturate the GPU’s massively parallel architecture, and the runtime is dominated by fixed overheads such as kernel-launch latency and device synchronization.
As the instance size increases, the GPU amortizes these fixed costs, whereas the runtime remains nearly constant until the hardware becomes fully saturated. 
In contrast, the CPU exhibits approximately linear growth in runtime as the instance size increases.
For the largest instance category, the GPU achieves a 6.67$\times$ speedup over the CPU, demonstrating the scalability advantages of GPU-accelerated gradient computation.
The detailed result is provided in the Supplementary Note~\ref{note:gradient}.

We next evaluate \textsc{FourierSMT} on two classes of combinatorial optimization problems, scheduling and placement, which are widely used as benchmarks due to their computational hardness and constraint nature~\cite{ansotegui2011satisfiability,saif2016pareto}. For both problem classes, SMT solvers are considered to be state-of-the-art. 
To better reflect practical scenarios, we design benchmark instances in which constraints are encoded with Boolean and real-valued variables, thereby extending beyond purely discrete formulations.
We compare \textsc{FourierSMT} against state-of-the-art SMT solvers in two benchmark suites comprising 380 instances.

\noindent\textbf{Scheduling Problems.}
We consider a continuous-time variant of the scheduling problem, where a number $n_j$ of jobs are assigned to $n_w$ workers, and execution can begin at any real-valued time stamp rather than at discrete intervals. This formulation remains NP, but generalizes the problem to more practical scenarios involving continuous scheduling. 
The goal is to find the assignment of all jobs within the given time budget. 
A full description of the problem is given in the Methods.

Benchmark instances were generated by varying the number of workers $n_w$ and the ratio between jobs and workers $r=n_j / n_w$.  
Specifically, we set $n_w \in \{16, 32, 64, 128, 256\}$ and chose $r \in \{2,3,4,5\}$. 
We generate 10 instances for each $n_w$ and $r$. 
Precedence constraints among jobs capture task dependencies. 
For each configuration, we created ten benchmark instances, with time budgets specified by a greedy heuristic. 

Fig.~\ref{fig:overall}a summarizes the results, where \textsc{FourierSMT} achieves the lowest PAR-2 score. 
The strongest baseline, \textsc{SMTS}, records 259.01 seconds, whereas \textsc{FourierSMT} is 6.34$\times$ faster.  
Relative to widely deployed solvers, speedups are even more pronounced: 14.2$\times$ over \textsc{Yices2}, 14.5$\times$ over \textsc{Z3++}, 16.3$\times$ over \textsc{Z3}, 22.4$\times$ over \textsc{CVC5}, 23.8$\times$ over \textsc{MathSAT5}, and 41.1$\times$ over \textsc{SMT-RAT}.  

Notably, \textsc{SMTS} is fastest on 79 out of 180 instances, followed closely by \textsc{FourierSMT} on 76, while the latter dominates the larger-scale instances.
\textsc{Yices2} and \textsc{Z3} lead on 38 and 8 instances, respectively. 
In small- to medium-scale instances, \textsc{FourierSMT} exhibits a warm-up phase as a result of exploration of the local objective landscape, leading to slightly slower runtimes.  
Crucially, on the largest instances, where other solvers fail to scale, \textsc{FourierSMT} solves all cases within the 1000-second time limit (see per-instance results in the Supplementary Note~\ref{note:detail}, Table~\ref{tab:scheduling}).  

To further assess generality, we next consider another NP-complete problem with different structural characteristics: the placement problem.  

\noindent\textbf{Placement Problems.}
In VLSI physical design, placement assigns components to precise positions on chiplets subject to area and non-overlap constraints. 
Our benchmark focuses on the 3D variant, where modules may be stacked across multiple layers. The solution of the benchmark problem encodes the spatial coordinates of all modules subject to the constraints.

Instances were generated by varying the number of macros $n_m \in \{2,4,8,16,32,64\}$ and the number of layers $n_l \in \{2,4,8\}$.  
We generate 10 instances for each $n_m$ and $n_l$. 
Routing-aware constraints were applied: modules within the same macro were colocated, modules on the same layer were placed adjacently, and cross-layer modules were connected via through-silicon vias, thereby reducing downstream routing cost (see Methods).

We find that \textsc{FourierSMT} again delivers the lowest PAR-2 score, see Fig.~\ref{fig:overall}b. 
Interestingly, here \textsc{Z3} surpasses \textsc{SMTS} as the strongest CDCL($\mathcal{T}$)-based baseline. 
Its variant \textsc{Z3++}, which augments \textsc{Z3} with local search, shows mixed behavior: outperforming \textsc{Z3} in scheduling but falling behind in placement.
This behavior arises because, in placement instances where local search offers limited improvement, the additional search overhead reduces the time available for CDCL($\mathcal{T}$) reasoning, causing \textsc{Z3++} to underperform relative to pure \textsc{Z3}.
In general, \textsc{FourierSMT} is 3.15$\times$ faster than \textsc{Z3} and 3.54$\times$ faster than \textsc{Z3++}. 
Relative improvement are larger compared to other baselines: 7.4$\times$ over \textsc{SMTS}, 8.69$\times$ over \textsc{Yices2}, 17.95$\times$ over \textsc{CVC5}, 21.8$\times$ over \textsc{MathSAT5}, and 36.0$\times$ over \textsc{SMT-RAT}. 
It can be seen that \textsc{FourierSMT} excels primarily on the largest 30 instances, which are intractable for the other solvers (see per-instance results in the Supplementary Note~\ref{note:detail}, Table~\ref{tab:placement}).

In conclusion, on the problems of scheduling and placement, \textsc{FourierSMT} demonstrates superior scalability while retaining efficiency. 
In particular, \textsc{VBS1} achieves speedups of 8.18$\times$ and 4.93$\times$ over \textsc{VBS2}. 
This advantage arises from the use of continuous relaxations, Gaussian smoothing, and projected gradient methods to efficiently navigate high-dimensional search spaces, allowing effective handling of instances where traditional CDCL($\mathcal{T}$) heuristics stagnate. 
Furthermore, the \ac{CLS} framework allows ForierSMT to achieve shorter runtimes due to strong parallelizability on multi-core hardware, which is very challenging for CDCL($\mathcal{T}$)-based approaches.

\section*{Discussion}
We introduced \textsc{FourierSMT}, a continuous optimization framework to solve SMT problems with linear real arithmetic (\ac{LRA}) constraints. 
While most state-of-the-art solvers are based on conflict-driven clause learning, \textsc{FourierSMT} is a \ac{CLS}-based approach, which allows gradient-dynamics for SMT. 
Concretely, we encode the SMT clauses in piecewise-linear functions with an extended Walsh-Fourier expansion (xWFE). 
We then reduce the number of terms in the xWFE, which grow exponentially in the number of variables and atoms, by introducing an efficient data structure called extended binary decision diagrams (xBDDs). 
The output probability of such xBDDs has been proven to be equivalent to the expectation value of xWFE. 
This allows us to find solutions to the SMT with gradient-based search on the output distribution of xBDDs. 
To boost convergence, we smooth the xBDD output distribution with randomized rounding and Gaussian sampling techniques. We adopt an annealing schedule, which scans the global landscape first and later localizes the search to a suitable region. 
This approach was shown to efficiently search the landscape while maintaining the optimality of the solutions.
Benchmarking \textsc{FourierSMT} both on random SMT instances as well as combinatorial optimization problems, such as scheduling and 3D placement, reveals superior scalability, surpassing previous state-of-the-art solvers in particular for large problem instances.

A key advantage of the framework is its compatibility with modern parallel hardware: gradient-based procedures map naturally to GPUs, enabling substantial acceleration at scale.
This combination of theoretical novelty and computational efficiency suggests that \textsc{FourierSMT} can serve as a foundation for solving industrial combinatorial problems at scale.

For future directions, extending the \textsc{FourierSMT} framework to nonlinear real arithmetic (NRA) would allow one to address nonlinear atomic constraints that induce highly nonconvex feasible regions and generally lack tractable closed-form expectations under Gaussian smoothing. 
Such an extension would broaden the applicability to domains such as hybrid system verification~\cite{cimatti2012quantifier} and biological network analysis~\cite{mccaig2008process}, while offering a fresh perspective on scalable algorithms for NRA.
Moreover, generalizing Proposition~\ref{prop:qp} to accommodate nonlinear constraints would provide efficient projection mechanisms for complex feasible sets, thereby further integrating continuous optimization with symbolic reasoning.
In this way, extending \textsc{FourierSMT} to NRA would help establish a unifying framework that complements complete methods such as cylindrical algebraic decomposition~\cite{caviness2012quantifier} and MCSAT~\cite{jovanovic2013design}.

\section*{Methods}

\subsection*{Projected Gradient Descent}
In \ac{PGD}, each iteration updates the current point by following the negative gradient. If the updated point lies outside the feasible domain, it is mapped back by a projection operator:
\begin{align}
    (a', b') & = (a, b) - \eta \nabla \mathcal{C}_\sigma(a, b), \label{eq:pg1} \\
    (a, b)   & = \text{proj}_\mathcal{D}(a',b'). \label{eq:pg2}
\end{align}
where $\mathcal{D}$ denotes the feasible set. 
Next, we formally define the projection operation and analyze its computational properties.

\begin{definition}[Projection]\label{def:qp}
    Let $(a',b') \in \mathbb{R}^{n+m}$ denote the point obtained after a gradient step.  
    The projection of $(a',b')$ onto the feasible set $\mathcal{D}$ is defined as
    \begin{align*}
        \text{proj}_\mathcal{D}(a',b')= \argmin_{(a,b)} (\lVert a-a' \rVert^2 + \lVert b-b' \rVert^2 + I_{\mathcal{D}}(a,b))
    \end{align*}
    where $I_\mathcal{D}$ is an indication. $I_\mathcal{D}(a,b)$ is 0 (resp. $\infty$) if $(a,b)\in{}\mathcal{D}$ (resp. $\notin$).
\end{definition}

Based on Definition~\ref{def:qp}, we can establish that the projection problem is a convex \ac{QP}.
\begin{proposition}\label{prop:qp}
    The projection defined in Definition~\ref{def:qp} is a convex \ac{QP} problem.
    \begin{align*}
        \min_{a,b}  &\quad \sum_{i=1}^n\left(a_i^2 - 2a_i'a_i\right) + \sum_{i=1}^m\left(b_i^2 - 2b_i'b_i\right)\\
        \text{s.t.} &\quad -1 \le a_i \le 1, \forall i  \\
                    &\quad \sum_{k=1}^m q_{j,k}b_{k} - q_{j,0} \bowtie 0, \forall \text{unit atomic constraint}
    \end{align*}
    where a unit atomic constraint refers to a constraint that only has an atomic constraint without propositional variables.
\end{proposition}

Substituting Eq.~\eqref{eq:pg2} into Eq.~\eqref{eq:pg1} yields the \emph{projected gradient mapping} $g$:
\begin{equation}\label{eq:grad_map}
    g(a, b) = \frac{1}{\eta} ((a, b) - \text{proj}_\mathcal{D}((a, b) - \eta \nabla \mathcal{C}_\sigma(a, b))).
\end{equation}

\subsection*{Annealing with Weight Adaptation}
\begin{algorithm}[h]
\caption{CLS with Annealing}
\label{alg:annealing}
\KwIn{A constraint set $C$, initial constraint weights $w_c$'s, an initial point $(a, b)$, a decay factor $\beta$, a scaling base $\gamma$, a update iteration $\tau$, an annealing schedule $\{\sigma_1, \sigma_2, \ldots, \sigma_T\}$}
\KwOut{Result \textsc{SAT} or \textsc{Unknown}}
Initialize $h_c\gets 0$ for all constraint $c\in{}C$\;
\For{$\sigma_t \in \{\sigma_1, \sigma_2, \ldots, \sigma_T\}$}{
    Local minimum $(a, b) \gets \texttt{ProjGradDesc}(a, b, \sigma_t, w)$\;
    Rounded assignment $(x, y) \gets (\texttt{sgn}(a), b)$\;
    $\#$\textsc{Unsat} $\gets0$\;
    \For{$c\in{}C$}{
        $u_c \gets \frac{1}{2}f_c(x, y) + \frac{1}{2}$\;
        $h_c \gets \rho\cdot h_c + u_c$\;
        $\#$\textsc{Unsat}$\gets\#$\textsc{Unsat}$ + u_c$\;
        \If{$t\bmod \tau = 0$}{
            $w_c \gets w_c \cdot \gamma ^ {h_c}$\tcp*[r]{update weight}
            $h_c \gets 1$\tcp*[r]{reset history}
        }
    }
    \If{$\#$\textsc{Unsat} $= 0$}{
        \Return \textsc{SAT}\;
    }
}
\Return \textsc{Unknown}\;
\end{algorithm}
\textsc{FourierSMT} integrates an annealing strategy with an adaptive clause-weighting scheme based on exponential recency-weighted averaging (ERWA), as outlined in Algorithm~\ref{alg:annealing}.
The optimization proceeds in a smoothed objective landscape, where the sampling parameter $\sigma$ is gradually decreased according to a predefined annealing schedule.
At early iterations, a large $\sigma$ yields a highly smoothed objective, enabling efficient descent toward promising regions.
As $\sigma$ decreases, the optimization progressively sharpens around the original combinatorial objective, allowing the algorithm to refine candidate solutions.

To guide the search away from local minima, we maintain dynamic clause weights with ERWA.
At each iteration, for every constraint $c \in C$, \textsc{FourierSMT} checks whether the rounded assignment satisfies $c$. 
The exponentially decayed history $h_c$ is then updated as
\begin{equation*}
    h_c[t+1] = \rho \cdot h_c[t] + \left[f_c(\mathcal{R}(a), b) = 1\right],
\end{equation*}
where $f_c(\mathcal{R}(a)) = 1$ indicates that the rounded assignment of $a$ violates constraint $c$, and $\rho \in (0,1)$ is a decay rate controlling the memory of past violations. 
After a fixed number of iterations $T$, clause weights are updated multiplicatively as
\begin{align*}
    \omega_c[t] &\gets \omega_c[t] \cdot \gamma^{h_c[t]} \\
\end{align*}
with scaling base $\gamma > 1$. 
This update rule exponentially emphasizes recently falsified clauses while attenuating the influence of earlier ones, adaptively biasing the search toward persistent sources of violation.

\textsc{FourierSMT} repeatedly invokes Alg.~\ref{alg:annealing} with different random seeds until either a \textsc{SAT} result is returned or the time limit is reached. 
For the benchmark experiments, we set $\rho = 0.5$, $\gamma = 2$, and $\tau = 1$.

\subsection*{Projected Critical Point}
The $\epsilon$-projected critical point is a concept used in the convergence analysis of projected gradient methods — it refers to a point where the projected gradient mapping has small norm, \emph{i.e.}, the optimization has nearly reached a stationary point within the feasible set.

Formally, in the FourierSMT framework, the projected gradient mapping is defined in Eq.~\eqref{eq:grad_map}. 
Then, a point $(a^*, b^*)$ is said to be an $\epsilon$-projected critical point if
\begin{equation}\label{eq:critical}
    ||g(a^*, b^*)||^2\le\epsilon^2,
\end{equation}
meaning the norm of the projected gradient is smaller than a small threshold $\epsilon$, \emph{i.e.}, the optimization has nearly converged.

\subsection*{Evaluation Metrics}
The PAR-2 score is a standard metric that balances runtime efficiency and robustness by penalizing timeouts:  
\begin{equation}\label{eq:PAR2}
    \text{PAR-2} = \frac{1}{N}\sum_{i=1}^N \Big( t_i \cdot \mathbf{1}[t_i \leq T] + 2T \cdot \mathbf{1}[t_i > T] \Big),
\end{equation}
where $\mathbf{1}(\cdot)$ is the indicator function and $T$ denotes the time limit, set to 1000 seconds in our experiments.

\subsection*{Benchmark Problem Descriptions}
\noindent\textbf{Random Hybrid Constraints} 
The random benchmark consists of cardinality, not-all-equal, and parity constraints, where these constraints are called hybrid constraints. 
WFE-based CLS approaches, \emph{i.e.}, \textsc{FourierSAT}~\cite{kyrillidis2020fouriersat} and \textsc{FastFourierSAT}~\cite{cen2025massively}, achieve performance competitive with SOTA CDCL-based SAT solvers on random hybrid constraint problems.
As \textsc{FourierSMT} is an extended version of \textsc{FourierSAT} that natively supports real-valued variables, we construct the benchmark instances using the following types of constraints.
\begin{align*}
    \text{cardinality}: & \sum_{i=1}^n l_i \le k \\
    \text{not-all-equal}:& \neg\left(l_1=l_2=\cdots=l_n\right) \\
    \text{parity}: & \left(l_1\oplus l_2\oplus\cdots\oplus l_n\right)
\end{align*}
$l_i$ are the literals, which can be propositional variables $x_j$, inequalities with real-valued variables for \textsc{Reals} $\alpha_k$ or their's negation $\neg x_j$, $\neg \alpha_k$. 

For \textsc{FourierSMT}, the coefficients of the \ac{xWFE} are computed following~\cite{kyrillidis2020fouriersat}, as the hybrid constraints are symmetric (For non-symmetric constraints, we can construct the \ac{xBDD} as in~\cite{kyrillidis2021continuous}). 
For the SMT solver competitors, the hybrid constraints are encoded into disjunctive clauses, where each clause is the disjunction of literals.
The cardinality constraints are encoded using the cardinality encoding in PySAT~\cite{imms-sat18,itk-sat24}.
The NAE constraints are encoded using the standard two-clause encoding that eliminates the all-true and all-false assignments.
The XOR constraints use the encoding in~\cite{li2000integrating}.

\noindent\textbf{Scheduling with Continuous Time} 
The scheduling problem is a classical combinatorial optimization problem that involves assigning a set of jobs (or tasks) to a set of workers (or machines, processors, or other resources) subject to various constraints.
Constraint programming solvers, such as Google’s OR-Tools CP-SAT~\cite{perron_et_al:LIPIcs.CP.2023.3}, are highly effective in handling discrete logical constraints and have achieved state-of-the-art performance on many discrete scheduling benchmarks.
There exist numerous variants of the scheduling problem, differing in the characteristics of jobs, resources, and temporal dependencies, each leading to distinct formulations and computational challenges. 
In this study, we focus on continuous-time scheduling problems, where both discrete logical decisions and real-valued timing constraints coexist.
Such problems can be naturally encoded and solved by \ac{SMT} solvers, which integrate logical reasoning with \ac{LRA}.

We encode the problem with the following variables:
\begin{itemize}
    \item Worker index: for each job $j$, we use a variable $X_{id,j}\in\{0,...,n_w-1\}$ to encode which worker it is assigned to.
    $X_{id,j}$ is Booleanized as follows
    \begin{equation*}
        X_{id,j} = \sum_{i=0}^{\log_2(n_w)-1}2^i \cdot [x_{i,j}=True].
    \end{equation*}
    \item Job start time: job $j$ is launched at $y_j\in\mathbb{R}$.
\end{itemize}

We generate instances by the following parameters. 
\begin{itemize}
    \item Worker start time: worker will start admitting job after $d_{id}\in\mathbb{R}$ time, which is sampled from $\mathcal{U}(0,1)$. 
    \item Job run time: to complete job $j$ need $t_j\in\mathbb{R}$ time, which is sampled from $\mathcal{U}(0,1)$. 
    \item Cutoff time: all workers will stop after $d_{id}+T$ time, where $T$ is a solution from a greedy heuristic to ensure feasibility.
\end{itemize}

The problems include the following constraints
\begin{itemize}
    \item Task dependency: if the input of a job $j$ depends on the output of another job $j'$, then $j$ has to be executed after $j'$. 
    \begin{equation*}
        y_j \geq y_{j'} + t_{j'}
    \end{equation*}
    \item Non-overlap: one worker can run only one job at a time. 
    Hence, for every pair of jobs $(j, j')$, 
    \begin{align*}
        &\left(x_{0,j}\oplus{}x_{0,j'}\right)\vee \ldots \vee \left(x_{\log_2(n_w)-1,j}\oplus{}x_{\log_2(n_w)-1,j'}\right) \\
        &\vee (y_j - y_{j'} \geq t_{j'}) \vee (y_{j'} - y_j \geq t_j).
    \end{align*}
    \item Feasibility: the job can only run on a worker if the worker is active. 
    Hence, for each job $j$, 
    \begin{align*}
        &\left(\neg x_{0,j}\right) \vee \ldots \vee \left(\neg x_{\log_2(n_w)-1,j}\right) \\
        &\vee (y_j \geq d_{id}) \vee (y_j \leq d_{id} + T - t_j).
    \end{align*}
\end{itemize}

\noindent\textbf{3D Placement}. 
The placement problem is another canonical combinatorial problem, widely studied in electronic design automation (EDA).
It involves assigning a set of components (or modules) to specific spatial locations on a chip, board, or layout region while satisfying geometric and connectivity constraints.
In integrated circuit and 3D chiplet design, placement determines the exact coordinates of macros and standard cells to minimize wirelength, delay, and congestion, while avoiding overlaps among modules.
Traditional approaches in the EDA design flow are mostly heuristic approaches~\cite{lin2019dreamplace}. 
Recent advances in SMT have enabled efficient modeling of placement problems by combining Boolean logic for non-overlap and connectivity constraints with linear real arithmetic for geometric relationships~\cite{lee2020sp}.
In this study, we evaluate a 3D placement problem, where modules are distributed across multiple layers and interconnected through through-silicon vias, posing additional spatial and connectivity challenges.

The placement of a module is represented by Booleanized indices and real-valued coordinates:  
\begin{itemize}
    \item Macro index: for each module $j$, a variable $M_{id,j} \in \{0,\dots,n_m-1\}$ encodes the macro assignment, Booleanized as    
    \begin{equation*}
        M_{id,j} = \sum_{i=0}^{\log_2(n_m)-1}2^i \cdot [m_{i,j}=True].
    \end{equation*}
    \item Layer index: a variable $L_{id,j} \in \{0,\dots,n_\ell-1\}$ encodes the assigned layer, Booleanized as
    \begin{equation*}
        L_{id,j} = \sum_{i=0}^{\log_2(n_l)-1}2^i \cdot [l_{i,j}=True].
    \end{equation*}
    \item Coordinates: each module $j$ has continuous position variables $(x_j,y_j) \in \mathbb{R}^2$.  
\end{itemize}

Each instance contains:  
\begin{itemize}
    \item $n_m$ macros, each occupying $1 \times 1$ nm.  
    \item $n_m$ large process elements ($0.4 \times 0.4$ nm) and $n_m \cdot n_\ell$ small process elements ($0.2 \times 0.2$ nm).  
    \item $n_m$ large memory elements ($0.1 \times 0.1$ nm) and $\tfrac{n_m \cdot n_\ell}{2}$ small memory elements ($0.1 \times 0.05$ nm).  
\end{itemize}  

Randomization is applied by alternating the alignment of small memory elements along the $x$- or $y$-axis.
The dimensions of each module are denoted by $w_j$ (width along $x$) and $d_j$ (height along $y$).

The problems include the following constraints
Three classes of constraints are enforced:  
\begin{itemize}
    \item Routing-aware: each large process is paired with a large memory, and each small process with two small memories. 
    For associated modules $(j,j')$ and $\forall i\in[0,\log_2(n_m)-1]$,  
    \begin{align*}
        \Big\{&\left(m_{i,j}=m_{i,j'}\right),\quad (x_j - x_{j'} \le w_{j'}),\quad (x_{j'} - x_j \le w_j) \\
        &(y_j - y_{j'} \le h_{j'}), \quad (y_{j'} - y_j \le h_j)\Big\}
    \end{align*}
    \emph{i.e.}, they should be placed in the same macro and should be adjacent to each other. 
    Such that the length of the metal wire can be reduced in the preceding routing stage. 
    \item Non-overlap: modules sharing the same macro and layer must not overlap. For pairs $(j,j')$:  
    \begin{align*}
        & \left(m_{0,j}\oplus{}m_{0,j'}\right) \vee \left(m_{\log_2(n_m)-1,j}\oplus{}m_{\log_2(n_m)-1,j'}\right)\\
        & \left(l_{0,j}\oplus{}l_{0,j'}\right) \vee \left(l_{\log_2(n_l)-1,j}\oplus{}l_{\log_2(n_l)-1,j'}\right)\\
        &\vee (x_j - x_{j'} \ge w_{j'}) \vee (x_{j'} - x_j \ge w_j) \\
        &\vee (y_j - y_{j'} \ge h_{j'}) \vee (y_{j'} - y_j \ge h_j)
    \end{align*}
    \item Feasibility: all modules must remain within the macro boundary, 
    \begin{equation*}
        \left\{(x_j \ge 0), (x_j \le 1 - w_j), (y_j \ge 0), (y_j \le 1 - d_j)\right\}
    \end{equation*}
\end{itemize}

\subsection*{Experimental Setup}
All experiments were conducted on a high-performance computing node equipped with 2$\times$ AMD EPYC 9654 CPUs, 24$\times$96 GB of DRAM, and 8$\times$ NVIDIA L40S GPUs.
Baseline solvers were permitted to use up to 64 threads of the CPUs. 
Because SAT solvers are highly memory-intensive due to clause database management, naive parallelization via multi-threaded racing, without inter-thread clause sharing, typically yields limited performance gains. 
Consequently, \textsc{MathSAT5}, which lacks multithreading support, was executed in single-threaded mode.
Our proposed solver, \textsc{FourierSMT}, leverages GPU acceleration. 
Eight independent replicas were executed across the eight available GPUs, and we report the median runtimes to provide more robust results.
A timeout of 1000 seconds was imposed per benchmark instance for all solvers.

\section*{Data Availability}
The benchmark instances that support the results of this study are available from the corresponding author upon request. 

\section*{Code availability}
The algorithms are described in the main text and the supplementary notes.
The code is available from the corresponding author upon request. 

\section*{Acknowledgements}
This work was supported in part by the National Research Foundation under the Prime Minister’s Office, Singapore, through the Competitive Research Programme (project ID NRF-CRP24-2020-0002 and NRF-CRP24-2020-0003); in part by the Ministry of Education, Singapore, through the Academic Research Fund Tier 2 (project ID MOE-T2EP50221-0008); and in part by the National University of Singapore, through the Microelectronics Seed Fund (FY2024).

\section*{Author contributions}
Y.C. conceived the research idea, developed the theoretical framework, and implemented the methodology.
X.F. supervised the project and provided conceptual guidance.
Y.C. and D.E. designed the experiments, and X.F. configured the hardware environment for experimental evaluations.
All authors analyzed the experimental results and wrote the manuscript. 

\section*{Competing interests}
The authors declare no competing interests.

\section*{Correspondence}
Correspondence and requests for materials should be addressed to Yunuo Cen and Xuanyao Fong.

\bibliography{apssamp}

\appendix
\onecolumngrid
\clearpage

\renewcommand{\appendixname}{Supplementary Note}
\makeatother

\appendix
\setcounter{section}{0}
\renewcommand{\thesection}{\arabic{section}}
\renewcommand{\thesubsection}{\thesection.\arabic{subsection}}

\section{Walsh-Fourier Expansion}\label{note:WFE}
In this section, we introduce the Walsh-Fourier expansion (WFE, Thm.~\ref{thm:wfe}) and further extend it for adapting real variables (xWFE, Cor.~\ref{cor:wfe}).
The results in this section lead to Eq.~\eqref{eq:general-WFE} in the main text. 

\begin{theorem}[Walsh-Fourier Expansion~\cite{o2014analysis}]\label{thm:wfe}
    Given a function $f:\{\pm{}1\}^n\to\mathbb{R}$, there is a unique way of expressing $f$ as a multilinear polynomial, where each term corresponds to a subset of $[n]$, according to:
    \begin{equation*}
        f(x)=\sum_{S\subseteq[n]}\left(\hat{f}(S)\cdot\prod_{i\in{}S}x_i\right)
    \end{equation*}
    where $\hat{f}(S)\in\mathbb{R}$ is called Walsh-Fourier coefficient, given $S$ and computed as:
    \begin{equation*}
        \hat{f}(S)=\mathop{\mathbb{E}}_{x\sim\{\pm1\}^n}\left(f(x)\cdot\prod_{i\in{}S}x_i\right)
    \end{equation*}
\end{theorem}

To extend WFE-based CLS methods to SMT(LRA), we generalize Thm.~\ref{thm:wfe} as follows.
\begin{corollary}[Extended Walsh-Fourier Expansion, xWFE]\label{cor:wfe}
    Let $c$ be a constraint consisting of $n$ propositional variables and $k$ atoms.
    Then $c$ can be represented as a function $f_c : \{\pm1\}^n \times \mathbb{R}^m \to \{\pm1\}$, where $-1$ denotes \textsc{True} and $+1$ denotes \textsc{False}.
    There is a unique way of expressing $f_c$ as a piecewise multilinear polynomial, where each term corresponds to a pair of subsets of $[n]$ and $[k]$, according to:
    \begin{equation*}
        f_c\left(x, y\right)=\sum_{\substack{S\subseteq[n] \\ T\subseteq[k]}}\left(\hat{f}_c(S,T)\cdot\prod_{i\in{}S}x_i\cdot\prod_{i\in{}T}\delta_i(y)\right)
    \end{equation*}
    where $\hat{f}(S,T)\in\mathbb{R}$ is called Walsh-Fourier coefficient, given $S$, $T$, and computed as:
    \begin{equation*}
        \hat{f}_c(S,T)=\mathop{\mathbb{E}}_{\substack{x\sim\{\pm 1\}^n \\ \delta(y)\sim\{\pm 1\}^k}}
        \left(f_c(x,y)\cdot\prod_{i\in{}S}x_i \cdot \prod_{i\in{}T}\delta_i(y)\right)
    \end{equation*}
\end{corollary}

\begin{proof}[Proof of Corollary~\ref{cor:wfe}]
A constraint $c$ consists of at most $n$ propositional variables and $k$ atomic constraints. 
Let $z=\left(x,\delta(y)\right)$
\begin{equation}\label{eq:sup-composite}
    \tilde{f}_c(z)=f_c(x, y)
\end{equation}
be a Boolean function $\tilde{f}_c:\{\pm1\}^{n+k}\to\{\pm1\}$. 
By applying Theorem~\ref{thm:wfe} to $\tilde{f}_c$, we have
\begin{equation*}
    \tilde{f}_c(z) = \sum_{U\subseteq[n+k]}\left(\hat{f}_c(U)\cdot\prod_{i\in U}z_i\right)
\end{equation*}
with coefficient
\begin{equation*}
    \hat{f}(U)=\mathop{\mathbb{E}}_{z\sim\{\pm1\}^n}\left(\tilde{f}(z)\cdot\prod_{i\in{}S}z_i\right).
\end{equation*}

By substituting the above equations with Eq.~\eqref{eq:sup-composite}, we have
\begin{align*}
    f_c\left(x, y\right)
    & = \sum_{U\subseteq[n+k]}\left(\hat{f}_c(U)\cdot\prod_{i\in U}z_i\right) \\
    & =\sum_{\substack{S\subseteq[n] \\ T\subseteq[k]}}\left(\hat{f}_c(S,T)\cdot\prod_{i\in{}S}x_i\cdot\prod_{i\in{}T}\delta_i(y)\right)
\end{align*}
where $x_i=z_i$ and $\delta_i(y) = z_{i+n}$, and
\begin{equation*}
    \hat{f}_c(S,T)=\mathop{\mathbb{E}}_{\substack{x\sim\{\pm 1\}^n \\ \delta(y)\sim\{\pm 1\}^k}}
    \left(f_c(x,y)\cdot\prod_{i\in{}S}x_i \cdot \prod_{i\in{}T}\delta_i(y)\right)
\end{equation*}
\end{proof}
\newpage

\section{Expectation}\label{note:expectation}
In Supplementary Note~\ref{note:WFE}, we extend the WFE to adapt real variables in SMT.
The extension is called xWFE and is still discrete in nature. 
The objective function in the \ac{CLS} framework is the expectation of xWFEs. 
The results in this section lead to Eq.~\eqref{eq:expectation} and Eq.~\eqref{eq:expectation-Fw} in the main text.
Furthermore, we mathematically derive the analytical expression of $d_i(b)$.

\begin{lemma}\label{lmm:expectation}
    Let $f_c$ denote the WFE of a constraint $c$.  
    For any real point $(a, b) \in [-1,1]^n \times \mathbb{R}^m$, the expectation of $f_c$ with respect to the distribution $\mathcal{S}_a$ is
    \begin{align*}\begin{split}
         \mathop{\mathbb{E}}_{x\sim\mathcal{S}_a}[f_c(x,b)] 
         =& f_c(a,b) \\
         =&\sum_{\substack{S\subseteq[n] \\ T\subseteq[k]}}\left(\hat{f}_c(S,T)\cdot\prod_{i\in{}S}a_i\cdot\prod_{i\in{}T}\delta_i(b)\right).
    \end{split}\end{align*}
\end{lemma}

\begin{proof}[Proof of Lemma~\ref{lmm:expectation}]
\begin{align*}
      \mathop{\mathbb{E}}_{x\sim\mathcal{S}_a}[f_c(x, y)] 
     =& \sum_{x\in\{\pm1\}^n}f_c(x,\delta(y))\cdot\mathcal{S}_a(x)\\
     =& \sum_{x\in\{\pm1\}^n}\sum_{\substack{S\subseteq[n] \\ T\subseteq[k]}}\left(\hat{f}(S,T)\cdot\prod_{i\in{}S}x_i \cdot\prod_{i\in{}T}\delta_i(y)\right)\cdot \mathcal{S}_a(x)\\
     =& \sum_{\substack{S\subseteq[n] \\ T\subseteq[k]}} \hat{f}(S,T) 
     \cdot \sum_{x\in\{\pm1\}^n} \left(\prod_{i\in{}S} x_i\cdot\mathcal{S}_a(x)\right)
     \cdot \left(\prod_{i\in{}T}\delta_i(y)\right)  \\
     =& \sum_{\substack{S\subseteq[n] \\ T\subseteq[k]}} \hat{f}(S,T) 
     \cdot \mathop{\mathbb{E}}_{x\sim\mathcal{S}_a}\left[\prod_{i\in{}S} x_i\right]
     \cdot \left(\prod_{i\in{}T}\delta_i(y)\right)  \\
     =& \sum_{\substack{S\subseteq[n] \\ T\subseteq[k]}} \left(\hat{f}(S,T) 
     \cdot \prod_{i\in{}S} a_i
     \cdot \prod_{i\in{}T}\delta_i(y)\right)  \\
     =& f_c(a, y)
\end{align*}
\end{proof}

$d_i(b)$ is the expectation of $\delta_i(y)$ given $y\sim\mathcal{N}(b, \sigma^2I)$.
When $\delta_i(y)$ corresponds to a linear constraint of the form $q_i^\top y \leq q_{i,0}$, it can be written as
\begin{align*}
    d_i(b)
    &=-\mathop{\mathbb{P}}_{y\sim\mathcal{N}(b,\sigma^2I)}\left[\sum_{j=1}^m q_{i,j}y_j - q_{i,0}\le 0 \right] +\mathop{\mathbb{P}}_{y\sim\mathcal{N}(b,\sigma^2I)}\left[\sum_{j=1}^m q_{i,j}y_j - q_{i,0}> 0 \right]\\
    &=1-2\mathop{\mathbb{P}}_{y\sim\mathcal{N}(b,\sigma^2I)}\left[\sum_{j=1}^m q_{i,j}y_j - q_{i,0}\le 0 \right]
\end{align*}
Let $z=\sum_{j=1}^m q_{i,j}b_j - q_{i,0}$ and $s^2=\sigma^2\sum_{j=1}^m q_{i,j}^2$.
\begin{align*}
    d_i(z)
    &=1-2\mathop{\mathbb{P}}_{z'\sim\mathcal{N}(z,s^2)}\left[z'\le 0 \right]\\
    &=1-2\Phi\left(\frac{-z}{s}\right)\\
    &=1-2\cdot\left(\frac{1}{2}\left(1+\mathrm{erf}\left(\frac{-z}{s\sqrt{2}}\right)\right)\right)\\
    &=\mathrm{erf}\left(\frac{z}{s\sqrt{2}}\right)
\end{align*}
where $\Phi$ is the standard cumulative distribution function, and $\mathrm{erf}$ is the error function.
By substituting $b$ and $\sigma$ back to the equation, we have
\begin{align*}
    d_i(b)&=\mathrm{erf}\left(
        \frac{\sum_{j=1}^m q_{i,j}b_j-q_{i,0}}{\sqrt{2\sum_{j=1}^m q_{i,j}^2}\sigma}
    \right)
\end{align*}

And the differentiation with respect to $b_j$ is
\begin{align*}
    \frac{\partial d_i(b)}{\partial b_j} = \frac{\sqrt{2} q_{i,j}}{\sqrt{\pi\sum_{j=1}^m q_{i,j}^2}\sigma}
    \mathrm{exp}\left(
        -\frac{\left(\sum_{j=1}^m q_{i,j}b_j-q_{i,0}\right)^2}{\sqrt{2\sum_{j=1}^m q_{i,j}^2}\sigma}.
    \right)
\end{align*}

Theorem~\ref{thm:expectation} is the consequence of Lemma~\ref{lmm:expectation} with Gaussian sampling $y\sim\mathcal{N}(b, \sigma^2I)$.

\begin{theorem}\label{thm:expectation}
    For any point $(a,b) \in [-1,1]^n \times \mathbb{R}^m$, the expectation of $f_c(x,y)$ under $x \sim \mathcal{S}_a$ and $y \sim \mathcal{N}(b,\sigma^2 I)$ is
    \begin{align*}
        \mathop{\mathbb{E}}_{\substack{x\sim\mathcal{S}_a \\ y\sim\mathcal{N}(b,\sigma^2I)}} [f_c(x, y)] 
        =\sum_{\substack{S\subseteq[n] \\ T\subseteq[k]}}\left(\hat{f}_c(S,T)\cdot\prod_{i\in{}S}a_i\cdot\prod_{i\in{}T}d_i(b)\right).
    \end{align*}
    Moreover, when $\sigma^2 = 0$, the expectation reduces to
    \begin{align*}
        \mathop{\mathbb{E}}_{\substack{x\sim\mathcal{S}_a \\ y\sim\mathcal{N}(b,\sigma^2I)}} [f_c(x, y)] 
         &= f_c(a,b) 
    \end{align*}
\end{theorem}

\begin{proof}[Proof of Theorem~\ref{thm:expectation}]

$x$ is subject a binary ditribution $\mathcal{S}_a$ and $y$ is subject to Gaussian distribution $\mathcal{N}(b,\sigma^2I)$. 
For the expectation on the joint distribution, we have:
\begin{align*}
    \mathop{\mathbb{E}}_{\substack{x\sim\mathcal{S}_a \\ y\sim\mathcal{N}(b,\sigma^2I)}} [f_c(x, y)] 
     =& \int_{y\in\mathbb{R}^m}\sum_{x\in\{\pm1\}^n}f_c(x,\delta(y))\cdot\mathcal{S}_a(x)\cdot dy \\
     =& \int_{y\in\mathbb{R}^m}\sum_{\substack{S\subseteq[n] \\ T\subseteq[k]}} \left(\hat{f}(S,T) 
     \cdot \prod_{i\in{}S} a_i
     \cdot \prod_{i\in{}T}\delta_i(y)\right)\cdot dy \text{\qquad(Due to Lemma~\ref{lmm:expectation})}\\
     =& \sum_{\substack{S\subseteq[n] \\ T\subseteq[k]}} \left(\hat{f}(S,T) 
     \cdot \prod_{i\in{}S} a_i
     \cdot \prod_{i\in{}T}\int_{y\in\mathbb{R}^m}\delta_i(y)\cdot dy\right)\\
     =& \sum_{\substack{S\subseteq[n] \\ T\subseteq[k]}} \left(\hat{f}(S,T) 
     \cdot \prod_{i\in{}S} a_i
     \cdot \prod_{i\in{}T}\mathop{\mathbb{E}}_{y\sim\mathcal{N}(b,\sigma^2I)}[\delta_i(y)]\right)\\
     =& \sum_{\substack{S\subseteq[n] \\ T\subseteq[k]}} \left(\hat{f}(S,T) 
     \cdot \prod_{i\in{}S} a_i
     \cdot \prod_{i\in{}T} d_i(b)\right)\\
\end{align*}
\end{proof}
\newpage

\section{Decision Diagrams}\label{note:DD}
To explicitly evaluate the WFE of a constraint requires state enumeration. 
This section introduces decision diagrams, an efficient data structure for evaluating the constraints (via the equivalence between the \ac{COP} of the decision diagram and expectation of the \ac{xWFE}, as in Supplementary Note~\ref{note:WMI}). 

\textsc{GradSAT} avoids state enumeration by maintaining compact representations of Boolean functions using binary decision diagrams (BDDs)~\cite{kyrillidis2021continuous}.
A BDD represents Boolean constraints $\mathbb{B}^n \to \mathbb{B}$ as an ordered directed acyclic graph, where each node corresponds to a variable and each edge encodes a branching decision.
Every decision node tests a Boolean variable and has two outgoing edges, corresponding to the true and false outcomes.
To evaluate a BDD, variables are assigned values, and the diagram is traversed according to the corresponding branches until reaching a terminal node.
The Boolean value at the terminal node is the output of the function for that assignment. 

In the extended algebraic decision diagram (xADD) framework~\cite{sanner2011symbolic}, linear programming is often used to prune unreachable nodes (e.g., a branch following $y > 1$ cannot subsequently satisfy $y < 0$), thereby yielding a minimal decision diagram.
Motivated by the xADD, we generalize BDDs to handle continuous variables.
\ac{xBDD} provides a unified representation in which Boolean and real-valued variables are captured within a decision diagram through piecewise case statements.
However, in our setting, we do not employ such pruning for two reasons.
First, the canonicity of linear-programming-pruned diagrams remains an open question.
Second, in the presence of probabilistic sampling, deterministically unreachable nodes can still be visited with non-zero probability.
For example, if $y = 1.1$, samples drawn from $y' \sim \mathcal{N}(y, \sigma^2)$ may still yield $y' < 0$.
Therefore, we treat an XBDD analogously to a BDD, interpreting atomic constraints as propositional variables, while acknowledging that minimality and canonicity are not guaranteed in the extended setting.

\begin{definition}[Extended Binary Decision Diagram, xBDD]~\label{def:XBDD}
    An xBDD is a directed acyclic graph $G(V,E)$ representing a multivariate function $\mathbb{B}^n\times\mathbb{R}^m\to\mathbb{B}$.
    Each decision node $v \in V$ corresponds either to a Boolean variable or to a theory atom over real variables.  
    Each edge $e \in E$ is labeled $(v,\texttt{true})$ or $(v,\texttt{false})$, depending on the outcome of evaluating the decision node $v$.  
    A path $p$ in $G$ is a sequence of edges such that each decision node appears at most once, with the edge labels indicating the branch taken at that node.
    Formally, a path can be written as a finite subset of all possible decision outcomes:  
    \begin{equation*}
        p \subset \{(v_1, \texttt{true}), (v_1, \texttt{false}), (v_2, \texttt{true}), (v_2, \texttt{false}), \cdots\}
    \end{equation*}
\end{definition}
\newpage

\section{Circuit-Output Probability}\label{note:WMI}
In Supplementary Note~\ref{note:DD} we present the decision diagram data structure.
By running the message propagation algorithm (Supplementary Note~\ref{note:MP}) on the decision diagram, we can efficiently evaluate the \ac{COP}, which is introduced in this section.  
The result in this section leads to Eq.~\eqref{eq:objective} in the main text.

The \ac{COP} is defined as a special case of weighted model counting (WMC). 
To extend \ac{COP} to adapt the real variable, we further discuss weighted model integration (WMI). 

\begin{definition}[Weighted model counting~\cite{chavira2008probabilistic}]
    Let $f$ be a function over Boolean variables $x$.
    Let $\mathcal{W}: \{-1,1\}^{n} \to \mathbb{R}$ be a weight function. 
    The weighted model counting of $f_c$ with respect to $\mathcal{W}$ is defined as
    \begin{equation*}
        \text{WMC}(\mathcal{W}) = \sum_{\substack{x \in \{-1,1\}^n}} f(x,y) \cdot \mathcal{W}(x),
    \end{equation*}
    Here $\mathcal{W}$ assigns a weight to each complete assignment, which can be expressed as the product of literal weights associated with Boolean variables:
    \begin{equation*}
        \mathcal{W}(x) = \prod_{x_i=-1}\omega^T_i \cdot\prod_{x_i=1}\omega^F_i
    \end{equation*}
    where $\omega^T_i, \omega^F_i \in \mathbb{R}$ are the literal weights.
\end{definition}

When the weights satisfy $\omega^T_i + \omega^F_i = 1$, and $\omega^T_i, \omega^F_i \ge 0$, the weighted model counting problem reduces to the circuit-output probability problem.
And the weight function becomes an input probability function. 
The circuit-output probability (COP) of $f$ is defined as
\begin{align*}
    \text{COP}(\mathcal{S}_a) = \sum_{x \in \{\pm1\}^{n}} f(x)\cdot \mathcal{S}_a(x),
\end{align*}
where $\mathcal{S}_a(x)$ is the randomized rounding distribution over Boolean variables, \emph{i.e.}, Eq.~\eqref{eq:obj}.

\begin{definition}[Weighted model integration]
    Let $f$ be a function over Boolean variables $x$ and real variables $y$.
    Let $\mathcal{W}: \{-1,1\}^{n}\times\mathbb{R}^m \to \mathbb{R}$ be a weight function. 
    The weighted model integration of $f_c$ with respect to $\mathcal{W}$ is defined as
    \begin{equation*}
        \text{WMI}(\mathcal{W}) = \sum_{x \in \{-1,1\}^n}\int_{y\in\mathbb{R}^m} f(x,y) \cdot \mathcal{W}(x, y)dy,
    \end{equation*}
    and
    \begin{equation*}
        \mathcal{W}(x,y) = \prod_{x_i=-1}\omega^T_i \cdot\prod_{x_i=1}\omega^F_i \cdot \prod_{j=1}^m \omega_j(y_j).
    \end{equation*}
    $\omega_j$ is the weight of the real variable when taking a different value $y_j$. 
\end{definition}

To generalize circuit-output probability for weighted model integration, the weight of the real variable should satisfy 
\begin{equation*}
    \int_{y_j}w_j(y_j)dy_j=1,
\end{equation*}
and for $y_j\in\mathbb{R}$, $w_j(y_j)\ge0$. 

\begin{remark}[Circuit-output probability]
    Let $f: \{\pm1\}^n \times \mathbb{R}^m \to \{\pm1\}$ be a function over propositional variables $x$ and real variables $y$.
    Let $\mathcal{P}_\sigma : \{\pm1\}^{n+k} \to [0,1]$ denote the input probability distribution.  
    Then the circuit-output probability (COP) of $f'$ is defined as
    \begin{align*}
         \text{COP}(\mathcal{P}_\sigma) = \sum_{x \in \{\pm1\}^{n}} \sum_{z \in \{\pm1\}^{k}} f(x,y)\cdot \mathcal{P}_\sigma(x,z).
    \end{align*}
    Here, the input probability distribution factorizes as
    {\begin{align*}
        \mathcal{P}_\sigma(x, z) = \mathcal{S}_a(x)\cdot \mathop{\mathbb{P}}_{y\sim\mathcal{N}(b,\sigma^2I)}[\delta(y)=z]. 
    \end{align*}}
    where $\mathop{\mathbb{P}}_{y\sim\mathcal{N}(b,\sigma^2I)}[\delta(y)=z]$ is the probability that the atomic constraints $\delta(y)$ evaluate to $z$ under Gaussian sampling.
\end{remark}

Intuitively, the COP measures the probability that a constraint is satisfied under randomized Boolean rounding and Gaussian perturbation of real variables.

\begin{corollary}\label{cor:cop=wfe}
    For any real point $(a,b) \in [-1,1]^n \times \mathbb{R}^m$,  let $\mathcal{P}_\sigma$ denote the joint probability distribution induced by $\mathcal{S}_a$ and $\mathcal{N}(b,\sigma^2 I)$.  
    For an SMT(LRA) constraint $c$, the circuit-output probability satisfies
    \begin{equation*}
        \text{COP}_c(\mathcal{P}_\sigma) = \mathop{\mathbb{E}}_{\substack{x\sim\mathcal{S}_a \\ y\sim\mathcal{N}(b,\sigma^2I)}} \left[f_c(x, y)\right].
    \end{equation*}
    Let $\mathcal{C}_\sigma(a,b)=\sum_{c\in{}C}w_c\cdot\texttt{COP}_c(\mathcal{P}_\sigma)$, we have
    \begin{equation*}
        \mathcal{C}_\sigma(a,b) = 
        \mathop{\mathbb{E}}_{\substack{x\sim\mathcal{S}_a \\ y\sim\mathcal{N}(b,\sigma^2I)}} \left[{\mathcal{F}_w(x, y)}\right]
    \end{equation*}
    Particularly, when $\sigma=0$, the above reduced to 
    \begin{equation*}
        \mathcal{C}_\sigma(a,b) = \mathcal{F}_w(a, b)
    \end{equation*}
\end{corollary}

\begin{proof}[Proof of Corollary~\ref{cor:cop=wfe}]

By the definition of circuit-output probability, we have
\begin{align*}
\mathrm{COP}_c(\mathcal{P}_\sigma)
&= \sum_{x\in\{\pm1\}^n}\sum_{z\in\{\pm1\}^k} f_c(x,y)\cdot\mathcal{P}_\sigma(x,z) \\
&= \sum_{z\in\{\pm1\}^k}\sum_{x\in\{\pm1\}^n} f_c(x,y)\cdot\mathcal{S}_a(x)\cdot \mathop{\mathbb{P}}_{y\sim\mathcal{N}(b,\sigma^2I)}[\delta(y)=z]\\
&= \sum_{z\in\{\pm1\}^k}\mathop{\mathbb{E}}_{x\sim\mathcal{S}_a} \left[f_c(x,y)\right] \cdot \mathop{\mathbb{P}}_{y\sim\mathcal{N}(b,\sigma^2I)}[\delta(y)=z]\\
&= \sum_{z\in\{\pm1\}^k}f_c(a,y) \cdot \mathop{\mathbb{P}}_{y\sim\mathcal{N}(b,\sigma^2I)}[\delta(y)=z]
\end{align*}

Define $\tilde f_c(x,\delta(y))=f_c(x,y)$, we have

\begin{align*}
\mathrm{COP}_c(\mathcal{P}_\sigma)
&= \sum_{z\in\{\pm1\}^k}\tilde{f}_c(a,z) \cdot \mathop{\mathbb{P}}_{y\sim\mathcal{N}(b,\sigma^2I)}[\delta(y)=z] \\
&= \mathop{\mathbb{E}}_{y\sim\mathcal{N}(b,\sigma^2I)} \left[\tilde{f}_c(a,z)\right] \\
&= \mathop{\mathbb{E}}_{y\sim\mathcal{N}(b,\sigma^2I)} \left[f_c(a,y)\right] \\
&= \mathop{\mathbb{E}}_{\substack{x\sim\mathcal{S}_a \\ y\sim\mathcal{N}(b,\sigma^2I)}} \left[f_c(x,y)\right]
\end{align*}

For the weighted objective, by linearity of expectation,
\begin{align*}
\mathcal{C}_\sigma(a,b)
&= \sum_{c\in C} w_c\cdot \mathrm{COP}_c(\mathcal{P}_\sigma) \\
&= \sum_{c\in C} w_c\cdot \mathop{\mathbb{E}}_{\substack{x\sim\mathcal{S}_a \\ y\sim\mathcal{N}(b,\sigma^2I)}}\left[f_c(x,y)\right] \\
&= \mathop{\mathbb{E}}_{\substack{x\sim\mathcal{S}_a \\ y\sim\mathcal{N}(b,\sigma^2I)}}\Big[\sum_{c\in C} w_c \cdot f_c(x,y)\Big] \\
&= \mathop{\mathbb{E}}_{\substack{x\sim\mathcal{S}_a \\ y\sim\mathcal{N}(b,\sigma^2I)}}\big[\mathcal{F}_w(x,y)\big]. 
\end{align*}
\end{proof}

\newpage

\section{Message Propagation}\label{note:MP}
Direct evaluation or differentiation of $\mathbb{E}[f_c]$ is, however, intractable, as it requires enumerating $\mathcal{O}(2^n)$ Boolean assignments.
In contrast, message propagation on the associated xBDD (see Supplementary Notes~\ref{note:DD} and~\ref{note:WMI}) provides an efficient surrogate: both forward and backward traversals scale with the size of the decision diagram rather than with the exponential size of the Boolean domain.
Thus, message propagation yields a tractable upper bound on computational cost.

The procedure consists of two phases: forward propagation (Alg.~\ref{alg:forward}) and backward propagation (Alg.~\ref{alg:backward}).
In the forward pass, top-down messages are initialized at the root and recursively distributed along decision edges according to assignment-dependent branching probabilities.
The true terminal node accumulates the top-down message, which encodes the probability that constraint $c$ is satisfied under Gaussian sampling.
In the backward pass, bottom-up messages are initialized at the true terminal and propagated in reverse topological order.
Local sensitivity terms are then computed by contrasting the adjoint contributions of the true and false successors.

By leveraging message passing on xBDDs, we can evaluate the expectation of the objective function without explicit state enumeration. The following theorem formalizes the correctness and computational complexity of this procedure.
\begin{theorem}[Evaluation]\label{thm:forward}
    Alg.~\ref{alg:forward} returns the expectation of the objective function in Eq.~\eqref{eq:obj}.
    The complexity of running Alg.~\ref{alg:forward} scales at $\mathcal{O}\left(\sum_c|V_c|\right)$.
\end{theorem}

To prove Thm.~\ref{thm:forward}, we first reproduce Lem.~2 in~\cite{kyrillidis2021continuous}. 
\begin{lemma}\label{lmm:forward}
    Let $G_c(V_c, E_c)$ be an xBDD encoded constraint $c$ and run Alg.~\ref{alg:forward} on $G_c$. 
    For a real assignment $(a, b)\in[-1,1]^n\times\mathbb{R}^m$, we have
    \begin{align*}
        \mathbb{P}[x_i=-1] &= \frac{1-a_i}{2}, \\
        \mathbb{P}[\delta_i(b)=-1] &= \frac{1-d_i(b)}{2}.
    \end{align*}
    Then for each node $v\in{}V_c$, we have
    \begin{equation*}
        m_{td}[v] = {Pr}(G_c,v),
    \end{equation*}
    where ${Pr}(G_c,v)$ is the probability that the node $v$ is on the path generated by $x$ and $d(\cdot)$ on $G_c$.
\end{lemma}

\begin{proof}[Proof of Lemma~\ref{lmm:forward}]
    We prove the statement by structural induction on xBDD $G_c$.
    We first denote 
    \begin{equation*}
        z = \left[\begin{array}{cccccc}
             \mathbb{P}[x_1=-1] & ... & \mathbb{P}[x_n=-1] &
            \mathbb{P}[d_1(b)=-1] & ... & \mathbb{P}[d_k(b)=-1]
        \end{array}\right]
    \end{equation*}
    
    \noindent\textbf{Basic step}: For the root node $r$ of $G_c$, 
    \begin{equation*}
        m_{td}[v] = 1, 
    \end{equation*}
    where the statement holds because $r$ is on the path generated by every discrete assignment on $G_c$. \\
    
    \noindent\textbf{Inductive step}: For each non-root node $v$ of $G_c$, let ${par}_T(v)$ be the set of parents of $v$ with an edge labeled by \textsc{True} and ${par}_F(v)$ be the set of parents of $v$ with an edge labeled by \textsc{False} in xBDD $G_c$. 
    After Alg.~\ref{alg:forward} terminates, we have
    \begin{equation*}
        m_{td}[v] = \sum_{u\in{par}_T(v)} z_i \cdot m_{td}[u] + \sum_{u\in{par}_T(v)} (1-z_i) \cdot m_{td}[u] 
    \end{equation*}
    By inductive hypothesis, 
    \begin{align*}
        m_{td}[v] 
        &= \sum_{u\in{par}_T(v)}  z_i \cdot{Pr}(G_c,u) + \sum_{u\in{par}_T(v)}  (1-z_i) \cdot{Pr}(G_c,u) \\
        &=  {Pr}(G_c,v) 
    \end{align*}
    The last inequality holds because the events of reaching $v$ from different parents are exclusive.
\end{proof}

\begin{proof}[Proof of Theorem~\ref{thm:forward}]
The above Lemma shows that, after running Alg.~\ref{alg:forward}, we have
\begin{align*}
    m_{td}[V_c[T]] &= {Pr}(G_c,V_c[T]) \\
    & = {COP}_c(\mathcal{P}_\sigma) \text{\qquad(By definition of circuit-output probability)}
\end{align*}

\noindent\textbf{Complexity}: 
Alg.~\ref{alg:forward} traverse all xBDD $G_c$ for all $c\in{}C$. 
The topological sort takes $O(|V_c|+|E_c|)$ time by Kahn's algorithm. 
The message propagation takes $O(|E_c|)$ time. 
Since $|E_c|=2|V_c|$, \emph{i.e.}, all non-terminal nodes in xBDD only have two outgoing edges. 
Hence, the overall complexity scales at $O(|V_c|)$.
\end{proof}

\begin{algorithm}[t]
\caption{Forward Traversal of XBDDs}
\label{alg:forward}
\KwIn{A constraint set $C$, constraint weight $w_c$'s, a real assignment $(a,b) \in [-1,1]^n\times\mathbb{R}^m$, and sampling area $\sigma$}
\KwOut{Objective function value}
\For{$i \in \texttt{range}(n)$}{
    $p_i\gets\frac{1-a_i}{2}$\;
}
\For{$i \in \texttt{range}(m)$}{
    $p_{i+n}\gets\frac{1-d_i(b)}{2}$\;
}
$fval\gets0$\;
\For{$c\in{}C$}{
    Get the XBDD $G_c(V_c, E_c)$\;
    Initialize $m_{td}[v] \leftarrow 1$ if $v$ is the root node; else $0$\;
    Sort node index $V\gets\texttt{sorted}(V_c)$\;
    \For{$v \in V$}{
        $m_{td}[v.t] \leftarrow m_{td}[v.t] + p_{id(v)} \cdot m_{td}[v]$\;
        $m_{td}[v.f] \leftarrow m_{td}[v.f] + (1-p_{id(v)}) \cdot m_{td}[v]$\;
    }
    $fval\gets fval + w_c \cdot (1-2\cdot m_{td}[V_c[T]])$
}
\Return $fval$\;
\end{algorithm}
\begin{algorithm}[t]
\caption{Backward Traversal of XBDD}
\label{alg:backward}
\KwIn{A constraint set $C$, constraint weight $w_c$'s, a real assignment $(a,b) \in [-1,1]^n\times\mathbb{R}^m$, and sampling area $\sigma$}
\KwOut{Gradient of the objective function}
Run Alg. 1 to get forward message $m_{td}$\;
\For{$i \in \texttt{range}(1, n+m+1)$}{
    $grad_i \gets 0 $\;
}
\For{$c\in{}C$}{
    Get the XBDD $G_c(V_c, E_c)$\;
    Initialize $m_{bu}[v] \leftarrow 1$ if $v$ is the true node; else $0$\;
    Sort node index $V\gets\texttt{sorted}(V_c, \texttt{reverse}=\texttt{true})$\;
    \For{$v \in V$}{
        $m_{bu}[v] \leftarrow m_{bu}[v] + p_{id(v)} \cdot m_{bu}[v.t]$\;
        $m_{bu}[v] \leftarrow m_{bu}[v] + (1-p_{id(v)}) \cdot m_{bu}[v.f]$\;
        $\Delta = m_{td}[v]\cdot(m_{bu}[v.F] - m_{bu}[v.T])$\;
        \If{$id(v) \le n$} {
            $grad_{id(v)} \leftarrow grad_{id(v)} + w_c\cdot\Delta $\;
        } \Else {
            \For{$i\in\texttt{range}(1, m+1)$}{
                $grad_{i+n} \leftarrow grad_{i+n} + w_c\cdot\frac{\partial d_{id(v)}(b)}{\partial b_{i+n}}\cdot\Delta $\;
            }
        }
    }
}
\Return $grad$\;
\end{algorithm}

By running Alg.~\ref{alg:forward} on xBDDs, the expectation can be obtained from top-down messages accumulated at the true terminal nodes.
Darker colors correspond to larger message values, which in the example increase with successive gradient steps.
We now extend this analysis to differentiation. 
The following theorem establishes the correctness and complexity of computing gradients of the objective function.
\begin{theorem}[Differentiation]\label{thm:backward}
    Alg.~\ref{alg:backward} returns the gradient of the expectation of the objective function in Eq.~\eqref{eq:obj}.
    The complexity of running Alg.~\ref{alg:backward} scales at $\mathcal{O}\left(m\cdot\sum_c|V_c|\right)$.
\end{theorem}

To prove Thm.~\ref{thm:backward}, we first reproduce Lem.~4 in~\cite{kyrillidis2021continuous}. 
\begin{lemma}\label{lmm:backward}
    Let $G_c(V_c, E_c)$ be an xBDD encoded constraint $c$ and run Alg.~\ref{alg:backward} on $G_c$. 
    For a real assignment $(a, b)\in[-1,1]^n\times\mathbb{R}^m$, we have
    \begin{align*}
        \mathbb{P}[x_i=-1] &= \frac{1-a_i}{2}, \\
        \mathbb{P}[\delta_i(b)=-1] &= \frac{1-d_i(b)}{2}.
    \end{align*}
    Then, for each node $v\in{}V_c$, let $f_{c,v}$ be the subfunction corresponding to the sub-xBDD generated by regarding $v$ as the root. 
    The following holds:
    \begin{equation*}
        m_{bu}[v] = {COP}_{c,v}(P_\sigma).
    \end{equation*}
\end{lemma}

\begin{proof}[Proof of Lemma~\ref{lmm:backward}]
    We prove the statement by structural induction on xBDD $G_c$.
    We first denote 
    \begin{equation*}
        z = \left[\begin{array}{cccccc}
             \mathbb{P}[x_1=-1] & ... & \mathbb{P}[x_n=-1] &
            \mathbb{P}[d_1(b)=-1] & ... & \mathbb{P}[d_k(b)=-1]
        \end{array}\right]
    \end{equation*}
    
    \noindent\textbf{Basic step}: For the terminal nodes, we have
    \begin{align*}
        m_{bu}[V_c[T]] &= 1,\\
        m_{bu}[V_c[F]] &= 0.
    \end{align*}
    Since the subfunctions given by \textsc{True} and \textsc{False} are just Boolean constants, the statement holds. 
    
    \noindent\textbf{Inductive step}: For each non-terminal node $v$ of $G_c$, after Alg.~\ref{alg:backward} terminates, we have
    \begin{equation*}
        m_{bu}[v] = \frac{1-z_i}{2} \cdot m_{bu}[v.T] + \frac{1+z_i}{2} \cdot m_{bu}[v.F]
    \end{equation*}
    By inductive hypothesis, 
    \begin{align*}
        m_{bu}[v] 
        &= \frac{1-z_i}{2} \cdot COP_{c,v.T}(P_\sigma)+ \frac{1+z_i}{2} \cdot COP_{c,v.F}(P_\sigma) \\
        &= COP_{c,v}(P_\sigma)
    \end{align*}
\end{proof}

\begin{proof}[Proof of Theorem~\ref{thm:backward}]
With Lem.~\ref{lmm:forward} and Lem.~\ref{lmm:backward}, we are ready to prove Thm.~\ref{thm:backward}.
Given an xBDD $G_c$, we expand all terminal nodes such that each can have only one parent node. 
The expanded xBDD is called $G_{c'}$, which has a true terminal set $T$ and a false terminal set $F$. 
By repeating the proof of Lem.~\ref{lmm:forward} and Lem.~\ref{lmm:backward} on $G_{c'}$, for any non-terminal node $v\notin(T\cup F)$, we have
\begin{align*}
    m_{td}'[v] &= m_{td}[v] \\
    m_{bu}'[v] &= m_{bu}[v]
\end{align*}
For any true terminal node $v\in{}T$ and false terminal node $u\in{}F$, we have
\begin{align*}
    m_{bu}'[v] &= 1 \\
    m_{bu}'[u] &= 0
\end{align*}

\begin{align*}
    COP_c(\mathcal{P}_\sigma) 
    &=\mathbb{P}[f_c=-1|\mathcal{P}_\sigma]\\
    &= \sum_{v\in \{u:id(u)=i\}}{Pr}(G_c', v) \cdot \mathbb{P}[f_c=-1|v, \mathcal{P}_\sigma] \\
    &= \sum_{v\in \{u:id(u)=i\}}m_{td}'[v] \cdot m_{bu}'[v] \\
    &= \sum_{v\in \{u:id(u)=i\}}m_{td}[v] \cdot m_{bu}[v] \\
    &= \sum_{v\in \{u:id(u)=i\}}m_{td}[v] \cdot \left(z_i\cdot m_{bu}[v.T] + (1-z_i)\cdot m_{bu}[v.F]\right)
\end{align*}

Hence we have

\begin{align*}
    \frac{\partial COP_c(\mathcal{P}_\sigma)}{\partial z_i}
    &= \sum_{v\in \{u:id(u)=i\}}m_{td}[v] \cdot \left( m_{bu}[v.T] - m_{bu}[v.F]\right) \\
    \frac{\partial f_c(a, b)}{\partial z_i}
    &= \sum_{v\in \{u:id(u)=i\}}m_{td}[v] \cdot \left( m_{bu}[v.F] - m_{bu}[v.T]\right)
\end{align*}

\noindent\textbf{Complexity}: 
Alg.~\ref{alg:forward} traverses all xBDD $G_c$ for all $c\in{}C$. 
The topological sort takes $O(|V_c|+|E_c|)$ time by Kahn's algorithm. 
The message propagation takes $O(|E_c|)$ time. 
For gradient computation, the worst case is that every literal is an atomic constraint that consists of all $m$ real variables. 
Therefore, the gradient computation takes $O(m\cdot|E_c|)$ time. 
Since $|E_c|=2|V_c|$, \emph{i.e.}, all non-terminal nodes in xBDD only have two outgoing edges. 
Hence, the overall complexity scales at $O(m\cdot|V_c|)$.
\end{proof}

\newpage

\section{Smoothness}\label{note:smoothness}
This section analyzes the Lipschitz continuity and local convexity of $\mathcal{C}_\sigma$, both of which are critical for establishing convergence guarantees.

\begin{proposition}[Lipschitz]\label{prop:lipschitz}
    Let $\alpha=\sum_c w_c$ and $\mathcal{C}_\sigma:[-1,1]^n\times\mathbb{R}^m\to[-\alpha, \alpha]$.
    Let $\beta$ denote the maximum number of occurrences of any single continuous variable within a constraint.  
    Let $\gamma=\frac{\sqrt{2}\beta}{\sqrt{\pi}\sigma}$.
    Then, for any two points $(a',b'), (a,b) \in [-1,1]^n \times \mathbb{R}^m$, we have
    \begin{align*}
        \left|\mathcal{C}_\sigma(a', b') - \mathcal{C}_\sigma(a, b)\right| 
            &\le \rho\left(\lVert a' - a \rVert^2 + \lVert b' - b \rVert^2\right)^\frac{1}{2} \\
        \|\nabla\mathcal{C}_\sigma(a', b') - \nabla\mathcal{C}_\sigma(a, b)\| 
            &\le L\left(\lVert a' - a \rVert^2 + \lVert b' - b \rVert^2\right)^\frac{1}{2} 
    \end{align*}
    where $\rho=\alpha\sqrt{n+m}\max\left(1,\gamma\right)$ and $L=\sqrt{n+m\gamma^2}\rho$.
    As a consequence of the Lipschitz continuity~\cite{nesterov2013introductory}, we have
    \begin{equation}\label{eq:lipschitz}
        \Big|\mathcal{C}_\sigma(a',b') - \mathcal{C}_\sigma(a,b) - \left<\nabla_a\mathcal{C}_\sigma(a,b), a'-a\right> 
        - \left<\nabla_b\mathcal{C}_\sigma(a,b), b'-b\right>\Big| \le \frac{L}{2} \left(\lVert a'-a\rVert^2 + \lVert b'-b\rVert^2\right)
    \end{equation}
\end{proposition}

\begin{proof}[Proof of Proposition~\ref{prop:lipschitz}]
\begin{align*}
    \nabla_{b_j}d_i(b) &= \frac{\sqrt{2}q_{i,j}}{\sqrt{\pi}\sigma\lVert q_i \rVert} \exp\left(
        -\frac{\left(\sum_{j=1}^mq_{i,j}b_j-q_{i,0}\right)^2}{2\sigma^2\lVert q_i\rVert^2}
    \right) \\
    \left|\nabla_{b_j}d_i(b)\right| &= \frac{\sqrt{2}\left|q_{i,j}\right|}{\sqrt{\pi}\sigma\lVert q_i \rVert} \exp\left(
        -\frac{\left(\sum_{j=1}^mq_{i,j}b_j-q_{i,0}\right)^2}{2\sigma^2\lVert q_i\rVert^2}
    \right) \\
    &\le \frac{\sqrt{2}\left|q_{i,j}\right|}{\sqrt{\pi}\sigma\lVert q_i \rVert} \\
    &\le \frac{\sqrt{2}}{\sqrt{\pi}\sigma}
\end{align*}

Therefore, we can have

\begin{align*}
    \nabla_{b_j}\text{COP}_c(a,b) & = \nabla_{b_j}\sum_{\substack{S\subseteq[n] \\ T\subseteq[k]}}\left(\hat{f}_c(S,T)\cdot\prod_{i\in{}S}a_i\cdot\prod_{i\in{}T}d_i(b)\right) \\
    & = \sum_{\substack{S\subseteq[n] \\ T\subseteq[k]}}\left(\hat{f}_c(S,T)\cdot\prod_{i\in{}S}a_i\cdot\nabla_{b_j}\left(\prod_{i\in{}T}d_i(b)\right)\right)\\
    & = \sum_{\substack{S\subseteq[n] \\ T\subseteq[k]}}\left(\hat{f}_c(S,T)\cdot
    \prod_{i\in{}S}a_i\cdot \left(\sum_{i\in{}T}\nabla_{b_j} d_i(b)\cdot\prod_{k\in{T},k\ne{}i}d_i(b)\right)\right) \\
\end{align*}

Suppose a real variable $b_j$ appears in the constraint at most $\beta$ times.
Equivalently, let $K=\{i:b_{i,j}\ne 0\}$, then $|K|\le\beta$. 
Then the above reduces to 
\begin{align*}
    \nabla_{b_j}\text{COP}_c(a,b) 
    & = \sum_{j'\in{}K}\sum_{\substack{S\subseteq[n] \\ T\subseteq[k]}}\left(\hat{f}_c(S,T)\cdot
    \prod_{i\in{}S}a_i\cdot \nabla_{b_j} d_{j'}(b)\cdot\prod_{i\in{T},i\ne{}j'}d_i(b)\right) \\
    & = \sum_{j'\in{}K}\nabla_{b_j}d_{j'}(b)\cdot\sum_{\substack{S\subseteq[n] \\ T\subseteq[k]}}\left(\hat{f}_c(S,T)\cdot
    \prod_{i\in{}S}a_i\cdot \prod_{i\in{T},i\ne{}j'}d_i(b)\right) \\
    & = \sum_{j'\in{}K}\nabla_{b_j}d_{j'}(b)\cdot\sum_{\substack{S\subseteq[n] \\ T\subseteq[k]}}\left(\hat{f}_c(S,T)\cdot
    \prod_{i\in{}S}a_i\cdot \prod_{i\in{T}}d_i(b)\right) \\
    \left|\nabla_{b_j}\text{COP}_c(a,b)\right| &\le \sum_{j'\in{}K}\underbrace{\left| \nabla_{b_j} d_i(b) \right|}_{\le \frac{\sqrt{2}}{\sqrt{\pi}\sigma}} \cdot \underbrace{\left|\sum_{\substack{S\subseteq[n] \\ T\subseteq[k]}}\left(\hat{f}_c(S,T)\cdot
    \prod_{i\in{}S}a_i \cdot\prod_{k\in{T},k\ne{}i}d_i(b)\right)\right|}_{\le 1} \\
    & \le \frac{\sqrt{2}\beta}{\sqrt{\pi}\sigma}\\
    & = \gamma.
\end{align*}

Then we have
\begin{align*}
    \left|\nabla_{b_j}\mathcal{C}_\sigma(a, b)\right| 
    & = \left|\sum_{c\in{}C}w_c\cdot\nabla_{b_j}\text{COP}_c(a,b) \right|\\
    & \le \sum_{c\in{}C}\left|w_c\cdot\nabla_{b_j}\text{COP}_c(a,b) \right|\\
    & \le \sum_{c\in{}C}\left|w_c\right|\cdot\left|\nabla_{b_j}\text{COP}_c(a,b) \right|\\
    & \le \alpha\gamma.
\end{align*}

Let $a^{(i)} = [a_1, \cdots, a_{i-1}, a_i', \cdots, a_n]$ and $b^{(i)} = [b_1, \cdots, b_{i-1}, b_i', \cdots, b_m']$, then we have:
\begin{align*}
            \left|\mathcal{C}_\sigma(a', b') - \mathcal{C}_\sigma(a, b)\right| 
        \le & \sum_{i=1}^n\left|\mathcal{C}_\sigma(a^{(i)}, b) - \mathcal{C}_\sigma(a^{(i-1)}, b)\right|  
          + \sum_{i=1}^m\left|\mathcal{C}_\sigma(a', b^{(i)}) - \mathcal{C}_\sigma(a', b^{(i-1)})\right| \\
        = & \sum_{i=1}^n \left|(a'_i-a_i) \nabla_{a_i}\mathcal{C}_\sigma(a^{(i-1)}, b)\right| + \sum_{i=1}^m \left|(b'_i-b_i) \nabla_{b_i}\mathcal{C}_\sigma(a, b^{(i-1)})\right| \\
        \le & \alpha\sum_{i=1}^n \left|a'_i-a_i\right| + \alpha\gamma\sum_{i=1}^m \left|b'_i-b_i\right| \\
        \le & \alpha\max\left(1,\gamma\right)\left(\sum_{i=1}^n \left|a'_i-a_i\right| + \sum_{i=1}^m \left|b'_i-b_i\right|\right) \\
        \le & \alpha\max\left(1,\gamma\right)\underbrace{\sqrt{n+m}\left(\lVert a'-a\rVert^2+\lVert b'-b\rVert^2\right)^\frac{1}{2}}_\text{Due to Cauchy–Schwarz inequality},
\end{align*}
where we denote $\rho=\alpha\sqrt{n+m}\max\left(1,\gamma\right)$.

Note that we can view $\nabla\mathcal{C}_\sigma(a, b)$ as a vector of $n+m$ multilinear-like functions. 

\begin{align*}
    \|\nabla\mathcal{C}_\sigma(a', b') - \nabla\mathcal{C}_\sigma(a, b)\|^2 
    &= \sum_{i=1}^n \left(\nabla_{a_i}\mathcal{C}_\sigma(a', b') - \nabla_{a_i}\mathcal{C}_\sigma(a, b)\right)^2+
    \sum_{j=1}^m \left(\nabla_{b_j}\mathcal{C}_\sigma(a', b') - \nabla_{b_j}\mathcal{C}_\sigma(a, b)\right)^2\\
    &\le \sum_{i=1}^n \rho^2\left((a'-a)^2+(b'-b)^2\right) + \sum_{j=1}^m \rho^2\gamma^2\left((a'-a)^2+(b'-b)^2\right) \\
    &= (n+m\gamma^2)\rho^2\left((a'-a)^2+(b'-b)^2\right) \\
    \|\nabla\mathcal{C}_\sigma(a', b') - \nabla\mathcal{C}_\sigma(a, b)\|
    & \le \rho\sqrt{n+m\gamma^2}\left(\lVert a'-a\rVert^2+\lVert b'-b\rVert^2\right)^\frac{1}{2}
\end{align*}

Then we have $L=\rho\sqrt{n+m\gamma^2}$.

\end{proof}

We next investigate the local convexity of $\mathcal{C}_\sigma$ and its descent property, which are essential for analyzing gradient-based methods. 

\begin{proposition}[Local convexity~\cite{nesterov2013introductory}]\label{prop:descent}
    Suppose $\mathcal{C}_\sigma$ is locally convex on a convex set $\Delta$.  
    Let $(a^*, b^*) = \argmin_{(a,b) \in \Delta} \mathcal{C}_\sigma(a,b)$.  
    Then, for all $(a,b) \in \Delta$,
    \begin{equation*}
        \left<\nabla_a\mathcal{C}_\sigma(a^*, b^*), a-a^*\right> + \left<\nabla_b\mathcal{C}_\sigma(a^*, b^*), b-b^*\right> \ge 0.
    \end{equation*}
\end{proposition}

The following property of the projected gradient mapping provides a key inequality for convergence analysis.  

\begin{proposition}[Inequality of gradient mapping~\cite{kyrillidis2021solving}]\label{prop:grad_map}
    If $\mathcal{C}_\sigma$ is locally convex on a convex set $\Delta$, then for every $(a,b) \in \Delta$,
    \begin{equation*}
        \left<\nabla\mathcal{C}_\sigma(a,b), g(a, b)\right> \ge \lVert g(a, b)\rVert^2.
    \end{equation*}
\end{proposition}

Equipped with the above results, we present convergence guarantees for our projected gradient approach in the following lemma and theorem.

\begin{lemma}[Descent Lemma]\label{lmm:descent}
    For the objective function $\mathcal{C}_\sigma$ under linear constraints, let $g$ denote its projected gradient.  
    If the step size satisfies $\eta \leq \frac{1}{L}$, then the projected gradient descent sequence obeys
    \begin{equation*}
        \mathcal{C}_\sigma\left(a^{(t+1)}, b^{(t+1)}\right) - \mathcal{C}_\sigma\left(a^{(t)}, b^{(t)}\right) \le -\frac{\eta}{2}\lVert g\left(a^{(t)}\right)\rVert^2
    \end{equation*}
\end{lemma}

\begin{proof}[Proof of Lemma~\ref{lmm:descent}]
    The proof takes the positive side of Eq.~\eqref{eq:lipschitz} to establish the gradient descent step.
    \begin{equation*}
        \mathcal{C}_\sigma(a,b) \le \mathcal{C}_\sigma(a',b') + \left<
            \nabla\mathcal{C}_\sigma(a',b'), \begin{bmatrix} a-a' \\ b-b' \end{bmatrix}
        \right> + \frac{L}{2}\left(\|a-a'\|^2+\|b-b'\|^2\right) 
    \end{equation*}
    Let $(a,b)=(a',b')-\eta{g}(a',b')$, the above inequality becomes
    \begin{equation*}
        \mathcal{C}_\sigma(a,b) \le \mathcal{C}_\sigma(a',b') - \eta\left<
            \nabla\mathcal{C}_\sigma(a',b'), {g}(a',b')
        \right> + \frac{\eta^2L}{2}\|{g}(a',b')\|^2.
    \end{equation*}
    Applying Proposition~\ref{prop:grad_map} and let $\eta\le{}\frac{1}{L}$, we then have
    \begin{align*}
        \mathcal{C}_\sigma(a,b)
        &\le \mathcal{C}_\sigma(a',b') - \eta\|g(a',b')\|^2 + \frac{\eta}{2}\|g(a',b')\|^2 \\
        &= \mathcal{C}_\sigma(a',b') - \frac{\eta}{2}\|g(a',b')\|^2.
    \end{align*}
\end{proof}

Lemma~\ref{lmm:descent} implies that, if projected gradient descent is run for sufficient iterations, the sum of gradient norms is always bounded by a constant, which further implies convergence to a stationary point, \emph{i.e.}, Theorem~\ref{thm:max_step}.

\begin{theorem}[Convergence Speed]\label{thm:max_step}
    With step size $\eta=\frac{1}{L}$, the projected gradient descent converges to an $\epsilon$-projected-critical point $(a^*, b^*)$ in $\mathcal{O}(\frac{\alpha L}{\epsilon^2})$ iterations.
\end{theorem}

\begin{proof}[Proof of Theorem~\ref{thm:max_step}]
From Lemma~\ref{lmm:descent}, we can assume $\eta=\frac{1}{L}$. 
Consider a sequence of points $a_{(t)}$'s and $b_{(t)}$'s such that $\left(a_{(t+1)},b_{(t+1)}\right) = \left(a_{(t)},b_{(t)}\right) + \eta g\left(a_{(t)},b_{(t)}\right)$.
Due to Lemma~\ref{lmm:descent} we have $\mathcal{C}_\sigma\left(a_{(t+1)},b_{(t+1)}\right)-\mathcal{C}_\sigma\left(a_{(t+1)},b_{(t)}\right)\le\frac{\eta}{2}\|g(a_{(t)},b_{(t)})\|^2$.
Given any initial point $a_{(0)}, b_{(0)}$, it will converge to $a^*, b^*$. 
After $T$ gradient steps, we have
\begin{align*}
    \mathcal{C}_\sigma(a^*, b^*) - \mathcal{C}_\sigma\left(a_{(0)},b_{(0)}\right)
    & \ge \frac{1}{2\alpha{}L}\sum_{t=0}^T \|g(a_{(t)},b_{(t)})\|^2\\
    & \ge \frac{T+1}{2\alpha{}L}\min_{t} \|g(a_{(t)},b_{(t)})\|^2,
\end{align*}
which can be rewritten as
\begin{align*}
    \min_{t} \|g(a_{(t)},b_{(t)})\|
    &\le \sqrt{\frac{2L}{T+1}} \sqrt{\mathcal{C}_\sigma(a^*, b^*) - \mathcal{C}_\sigma\left(a_{(0)},b_{(0)}\right)}\\
    &\le \sqrt{\frac{2L}{T+1}} \sqrt{2\alpha}
\end{align*}
To achieve a point $a_{(T)}$ such that $\|g(a_{(t)})\|\le\epsilon$, we require $\sqrt{\frac{4\alpha{}L}{T+1}}\le\epsilon$, which implies $T\sim \mathcal{O}\left(\frac{\alpha{}L}{\epsilon^2}\right)$.
\end{proof}

\newpage

\section{Local Lipschitz}\label{note:local_smoothness}
In the Supplementary Note~\ref{note:smoothness}, Prop.~\ref{prop:lipschitz}, we established global Lipschitz continuity of the smoothed objective and its gradient over the entire domain.
Here, we provide a complementary result showing that $\mathcal{C}_\sigma$ admits tighter local Lipschitz constants when the real-valued variables are away from the transition regions induced by the atomic constraints.
This refinement is useful for analyzing convergence behavior in the regions close to the minima.

\begin{corollary}[Local Lipschitz]\label{cor:lipschitz}
    For every atomic constraint $\alpha_i$, we define the transition region $r_i$ by
    \begin{equation*}
        \left\{ y \,\middle|\, \left(\sum_{j=1}^m q_{i,j} y_j - q_{i,0}\right)^2 < 
        \sigma^2\ln \frac{2\beta^2}{\pi\sigma^2} \right\},
    \end{equation*}
    and $\mathcal{R}=\bigcup_i r_i$. 
    For any two points points $(a',b'), (a,b)\in[-1,1]^n\times(\mathbb{R}^m\backslash\mathcal{R})$, we have
    \begin{align*}
        &\left|\mathcal{C}(a', b') - \mathcal{C}(a, b)\right| \le {\alpha\sqrt{n+m}}
        \left(\lVert a' - a \rVert + \lVert b' - b \rVert\right), \\
        &\|\nabla\mathcal{C}_\sigma(a', b') - \nabla\mathcal{C}_\sigma(a, b)\| 
            \le \alpha(n+m)\left(\lVert a' - a \rVert^2 + \lVert b' - b \rVert^2\right)^\frac{1}{2} 
    \end{align*}
\end{corollary}

\begin{proof}[Proof of Corollary~\ref{cor:lipschitz}]
The transition region $r_i$ is defined such that for any point $b\in\mathbb{R}^m\backslash r_i$, we have
\begin{equation*}
    \left|\nabla_{b_j}d_i(b)\right| \le \frac{1}{\beta}. 
\end{equation*}

And hence
\begin{align*}
    \frac{\sqrt{2}\left|q_{i,j}\right|}{\sqrt{\pi}\sigma\lVert q_i \rVert} \exp\left(
            -\frac{\left(\sum_{j=1}^mq_{i,j}b_j-q_{i,0}\right)^2}{2\sigma^2\lVert q_i\rVert^2}
        \right) & \le \frac{1}{\beta} \\
        \left(\sum_{j=1}^mq_{i,j}b_j-q_{i,0}\right)^2 
        &\ge -2\sigma^2\lVert q_i\rVert^2\ln \frac{\sqrt{\pi}\sigma\lVert q_i \rVert}{\sqrt{2}\beta\left|q_{i,j}\right|}\\
        & = \sigma^2\ln \frac{2\beta^2}{\pi\sigma^2}\\
\end{align*}

For any point $b\in\mathbb{R}^m\backslash r_i$, 
\begin{align*}
    \left|\nabla_{b_j}\text{COP}_c(a,b)\right| &\le \sum_{j'\in{}K}\underbrace{\left| \nabla_{b_j} d_i(b) \right|}_{\le \frac{1}{\beta}} \cdot \underbrace{\left|\sum_{\substack{S\subseteq[n] \\ T\subseteq[k]}}\left(\hat{f}_c(S,T)\cdot
    \prod_{i\in{}S}a_i \cdot\prod_{k\in{T},k\ne{}i}d_i(b)\right)\right|}_{\le 1} \\
    & \le \beta\cdot\frac{1}{\beta}\\
    & = 1\\
    \left|\nabla_{b_j}\mathcal{C}_\sigma(a, b)\right| 
    & \le \sum_{c\in{}C}\left|w_c\right|\cdot\left|\nabla_{b_j}\text{COP}_c(a,b) \right|\\
    & \le \alpha\cdot 1\\
    & = \alpha.
\end{align*}

For the objective function, we have

\begin{align*}
            \left|\mathcal{C}(a', b') - \mathcal{C}(a, b)\right|
        \le & \sum_{i=1}^n\left|\mathcal{C}(a^{(i)}, b) - \mathcal{C}(a^{(i-1)}, b)\right|  
          + \sum_{i=1}^m\left|\mathcal{C}(a', b^{(i)}) - \mathcal{C}(a', b^{(i-1)})\right| \\
        \le & \alpha\sum_{i=1}^n \left|a'_i-a_i\right| + \alpha\sum_{i=1}^m \left|b'_i-b_i\right| \\
        \le & \alpha\sqrt{n+m}\left(\lVert a'-a\rVert^2+\lVert b'-b\rVert^2\right)^\frac{1}{2}.
\end{align*}

For the gradient, we have

\begin{align*}
    \|\nabla\mathcal{C}_\sigma(a', b') - \nabla\mathcal{C}_\sigma(a, b)\|^2 
    &= \sum_{i=1}^n \left(\nabla_{a_i}\mathcal{C}_\sigma(a', b') - \nabla_{a_i}\mathcal{C}_\sigma(a, b)\right)^2+
    \sum_{j=1}^m \left(\nabla_{b_j}\mathcal{C}_\sigma(a', b') - \nabla_{b_j}\mathcal{C}_\sigma(a, b)\right)^2\\
    &\le \sum_{i=1}^n \alpha^2(n+m)\left((a'-a)^2+(b'-b)^2\right) + \sum_{j=1}^m \alpha^2(n+m)\left((a'-a)^2+(b'-b)^2\right) \\
    &= \alpha^2(n+m)^2\left((a'-a)^2+(b'-b)^2\right) \\
    \|\nabla\mathcal{C}_\sigma(a', b') - \nabla\mathcal{C}_\sigma(a, b)\|
    & \le \alpha(n+m)\left(\lVert a'-a\rVert^2+\lVert b'-b\rVert^2\right)^\frac{1}{2}
\end{align*}
\end{proof}
\newpage

\section{Optimality}\label{note:optimality}
Thm.~\ref{thm:optimality} is the main text states that, when $\sigma=0$, the projected gradient will converge to the boundary of the feasible domain. 
This section proves Thm.~\ref{thm:optimality} by finding a gradient direction of the objective function in the open domain. 
If the descent direction always exists, then the convergence point will be on the boundary of a closed domain. 

\begin{lemma}[Gradient Direction]\label{lmm:optimality}
    Let $f:\mathbb{R}^n \times \mathbb{R}^m \to \mathbb{R}$ be expressible as a multilinear-like polynomial:
    \begin{equation}\label{eq:multilinear}
        f(a, b) =\sum_{\substack{S\subseteq[n] \\ T\subseteq[k]}}\left(\hat{f}(S,T)\cdot\prod_{i\in{}S}a_i\cdot\prod_{i\in{}T}\delta_i(y)\right),
    \end{equation}
    where $\hat{f}(S, T)$'s are the Walsh-Fourier coefficients and at least one coefficient is nonzero.  
    
    Then for every point $(a,b) \in \mathbb{R}^n \times \mathbb{R}^m$, there exist a region size $\mathcal{E} > 0$ and two directions $v^+, v^- \in \mathbb{R}^n$ such that for all $\epsilon \in (0,\mathcal{E})$,
    \begin{align*}
        f(a+\epsilon v^+, b) > f(a, b), \\
        f(a+\epsilon v^-, b) < f(a, b).
    \end{align*}
\end{lemma}

\begin{proof}[Proof of Lemma~\ref{lmm:optimality}]
Lemma 3 in~\cite{kyrillidis2021solving} has shown that any multilinear function $h$ and a point $a\in\mathbb{R}^n$, there exist a region size $\epsilon>0$ and two direction $v^+, v^-\in\mathbb{R}^n$, such that for all $\epsilon\in(0, \epsilon^+)$, the following holds
\begin{align*}
    h(a+\epsilon v^+) > h(a) \\
    h(a+\epsilon v^-) < h(a)
\end{align*}

Given the multilinear-like function $f$ and any point $(a,b)\in\mathbb{R}^n\times\mathbb{R}^m$, the function can be written as 
\begin{align*}
    g(a) &= f(a, b) \\
    &= \sum_{\substack{S\subseteq[n] \\ T\subseteq[k]}}\left(\hat{f}(S,T)\cdot\prod_{i\in{}T}\delta_i(b)\right)\cdot\prod_{i\in{}S}a_i \\
    &= \sum_{S\subseteq[n]}\hat{g}_b(S)\cdot\prod_{i\in{}S}a_i
\end{align*}
where $\hat{g}_b(S)$ is the coefficient. 
\begin{equation*}
    \hat{g}_b(S)=\sum_{S\subseteq[T]}\hat{f}(S,T)\cdot\prod_{i\in{}T}\delta_i(b)
\end{equation*}
If the coefficient $\hat{f}(S,T)$ is not 0, then $\hat{g}(S)$ will not be 0 since $\delta_i(b)\in\{\pm 1\}$. 
When $g(a)$ is a non-constant multilinear function of $a$, and we can have
\begin{align*}
    f(a+\epsilon v^+, b) > f(a, b), \\
    f(a+\epsilon v^-, b) < f(a, b).
\end{align*}
\end{proof}

Lemma~\ref{lmm:optimality} implies that Eq.~\eqref{eq:multilinear} admits no local optima within the open domain $(-1,1)^n \times \mathbb{R}^m$.
This directly results in Theorem~\ref{thm:optimality}.

\newpage

\newpage
\section{Technical Proofs}\label{note:proof}
Some proofs are provided in the supplementary notes where the corresponding results first appear.
This section presents the proofs of the theoretical results stated in the main text.

\subsection*{Proof of Theorem~\ref{thm:reduction}}
The \ac{CLS} approach relies on Thm.~\ref{thm:reduction} in the main text.
This subsection proves Thm.~\ref{thm:reduction}.

Each atomic constraint partitions the real space into finitely many polyhedral regions $a_1, \ldots, a_N$, the objective $\mathcal{F}_w$ admits the case representation
\begin{equation*}
    f_c(x, y)=
    \begin{cases}
        f_c(x|y\in a_1)& \text{if } y\in a_1, \\
        & \vdots \\
        f_c& \text{if } y\in a_N. \\
\end{cases}
\end{equation*}
Let $f_c^{(i)}(x) = f_c(x|y\in a_i)$, which is a multilinear polynomial in $x$. 

The Lem. 1 in~\cite{kyrillidis2021solving} states that, if $x\in[-1,1]$, for $i\in\{1, ..., N\}$, we have $f_c^{(i)}(x)\in[-1,1]$. 
Since the above statement holds in each case, it immediately yields the following Lemma.
\begin{lemma}
    If $(x,y)\in[-1,1]^n\times\mathbb{R}^m$, we have $f_c(x, y)\in[-1,1]$. 
\end{lemma}

Now we are ready to prove Thm.~\ref{thm:reduction}.

\begin{proof}[Proof of Theorem~\ref{thm:reduction}]
\text{}

\begin{itemize}
    \item ``$\Rightarrow$": Suppose $F=\bigwedge_{c\in C}c$ is satisfiable and $(x,y)\in\{\pm1\}^n\times\mathbb{R}^m$ is one of its solutions. 
    Then $(x,y)$ is also a solution for every constraint $c$. 
    Thus, for every $c\in{}C$, $f_c(x, y)=-1$. 
    Therefore $\mathcal{F}_w(x,y)=-\sum_{c\in C}w_c$ and
    \begin{equation*}
        \min_{\substack{x\in[-1,1]^n \\ y\in\mathbb{R}^m}}\mathcal{F}_w(x,y) = -\sum_{c\in{}C}w_c.
    \end{equation*}
    \item ``$\Leftarrow$": Suppose 
    \begin{equation*}
        \min_{\substack{x\in[-1,1]^n \\ y\in\mathbb{R}^m}}\mathcal{F}_w(x,y) = -\sum_{c\in{}C}w_c.
    \end{equation*}
    There must exists $(x,y)\in\{\pm1\}^n\times\mathbb{R}^m$ such that $\mathcal{F}_w(x,y) = -\sum_{c\in{}C}w_c$. 
    By the above Lemma, $(x,y)$ is a solution of $f_c$ for every $c\in{}C$. 
    Thus $(x,y)$ is also a solution of $F_w$. 
\end{itemize}
\end{proof}

\subsection*{Proof of Proposition~\ref{prop:qp}}
Our \ac{CLS} approach relies on projected gradient descent, \emph{i.e.}, the complexity of the projection heavily affects the performance of CLS.
This section proves that the projection in our \ac{CLS} approach is a convex problem. 

\begin{proof}[Proof of Proposition~\ref{prop:qp}]
Firstly, the minimization objective function in the matrix form is
\begin{equation*}
\begin{bmatrix} a \\ b \end{bmatrix}^{\!\top}
    \underbrace{\begin{bmatrix} I_n & 0 \\ 0 & I_m \end{bmatrix}}_\mathrm{Hessian}
        \begin{bmatrix} a \\ b \end{bmatrix}
        -2 \begin{bmatrix} (a')^\top & (b')^\top \end{bmatrix}
        \begin{bmatrix} a \\ b \end{bmatrix},
\end{equation*}
up to an irrelevant constant.
The Hessian is $\mathrm{diag}(I_n, I_m)$, which is an identity matrix, \emph{e.g.}, positive definite and convex. 
Therefore, the minimization objective function is a convex function.

Secondly, the box constraints $-1 \leq a_i \leq 1, \forall i$ define a closed halfspace. 
The unit theory atoms $\sum_{k=1}^m q_{j,k}b_{k} - q_{j,0} \bowtie 0$ are linear inequalities or equalities. 
The intersection of two convex sets is, therefore, convex.
\end{proof}

\newpage
\section{Per-Instance Benchmark Results}\label{note:detail}
The main text summarizes the overall performance across the three benchmark suites.
Here, we present the complete per-instance runtime results for all solvers.
\textsc{FourierSMT} is executed using 8 GPUs in parallel, and the median runtime per instance is reported in the main text.
To complement this, we first report the per-GPU runtime statistics in Table~\ref{tab:random-gpu}, Table~\ref{tab:scheduling-gpu}, and Table~\ref{tab:placement-gpu}.
We then provide the full per-instance runtimes of all solvers in Table~\ref{tab:random}, Table~\ref{tab:scheduling}, and Table~\ref{tab:placement}.
Eventually, we summarize the results of Table~\ref{tab:scheduling}, and Table~\ref{tab:placement} into Fig.~\ref{fig:by-instance}.

\begin{longtable}{%
p{0.10\textwidth}%
p{0.12\textwidth}%
p{0.10\textwidth}%
p{0.10\textwidth}%
p{0.10\textwidth}%
p{0.10\textwidth}%
p{0.10\textwidth}%
p{0.10\textwidth}%
p{0.10\textwidth}%
}
\caption{\textbf{Per-instance runtime of \textsc{FourierSMT} for the random hybrid-constraint benchmark.}
 Eight GPUs use different seed numbers. 
 All runtimes are reported in seconds. 
 The label “to” indicates that the solver exceeded the allotted time limit (1000 seconds) and did not return a solution.}
 \label{tab:random-gpu}\\
\hline
Instances & 0 & 1 & 2 & 3 & 4 & 5 & 6 & 7 \\
\hline
100\_0 & 9.981 & 8.530 & 9.122 & 8.417 & 10.139 & 8.966 & 8.766 & 10.231 \\
100\_1 & 10.226 & 10.005 & 10.071 & 8.736 & 10.225 & 10.220 & 10.144 & 11.852 \\
100\_2 & 10.494 & 10.311 & 10.074 & 12.362 & 10.259 & 10.953 & 10.812 & 12.552 \\
100\_3 & 12.856 & 11.400 & 10.109 & 13.103 & 10.335 & 10.980 & 12.991 & 14.721 \\
100\_4 & 13.021 & 12.790 & 10.588 & 14.664 & 10.339 & 12.563 & 13.171 & 14.824 \\
100\_5 & 14.094 & 13.071 & 12.362 & 14.689 & 10.428 & 13.786 & 13.573 & 14.862 \\
\hline
\end{longtable}

\begin{longtable}{%
P{0.10\textwidth}%
P{0.12\textwidth}%
P{0.10\textwidth}%
P{0.10\textwidth}%
P{0.10\textwidth}%
P{0.10\textwidth}%
P{0.10\textwidth}%
P{0.10\textwidth}%
P{0.10\textwidth}%
}
\caption{\textbf{Per-instance runtime of \textsc{FourierSMT} for the continuous time scheduling benchmark.}
  Eight GPUs use different seed numbers. 
 All runtimes are reported in seconds. 
 The label “to” indicates that the solver exceeded the allotted time limit (1000 seconds) and did not return a solution.}
 \label{tab:scheduling-gpu}\\
\hline
Instances & 0 & 1 & 2 & 3 & 4 & 5 & 6 & 7 \\
\hline
16\_2\_0 & 27.110 & 28.163 & 27.517 & 27.472 & 26.969 & 24.363 & 21.773 & 16.021 \\
16\_2\_1 & 15.058 & 17.114 & 13.984 & 29.762 & 16.877 & 16.968 & 13.735 & 17.574 \\
16\_2\_2 & 12.492 & 13.120 & 12.740 & 12.345 & 12.997 & 13.266 & 13.477 & 12.797 \\
16\_2\_3 & 12.292 & 12.423 & 11.667 & 22.216 & 11.524 & 12.735 & 11.093 & 27.390 \\
16\_2\_4 & 9.836 & 11.120 & 11.540 & 10.510 & 11.032 & 11.607 & 12.166 & 11.090 \\
16\_2\_5 & 8.723 & 9.064 & 10.555 & 9.773 & 9.173 & 9.884 & 10.575 & 8.529 \\
\hline
\end{longtable}

\begin{longtable}{%
P{0.10\textwidth}%
P{0.12\textwidth}%
P{0.10\textwidth}%
P{0.10\textwidth}%
P{0.10\textwidth}%
P{0.10\textwidth}%
P{0.10\textwidth}%
P{0.10\textwidth}%
P{0.10\textwidth}%
}
\caption{\textbf{Per-instance runtime of \textsc{FourierSMT} for the 3D placement benchmark.}
 Eight GPUs use different seed numbers. 
 All runtimes are reported in seconds. 
 The label “to” indicates that the solver exceeded the allotted time limit (1000 seconds) and did not return a solution.}
 \label{tab:placement-gpu}\\
\hline
Instances & 0 & 1 & 2 & 3 & 4 & 5 & 6 & 7 \\
\hline
2\_2\_0 & 13.273 & 9.910 & 9.556 & 7.955 & 7.523 & 9.020 & 8.813 & 7.594 \\
2\_2\_1 & 9.251 & 12.525 & 11.632 & 7.698 & 7.608 & 8.719 & 7.870 & 7.994 \\
2\_2\_2 & 8.830 & 8.555 & 8.196 & 8.370 & 7.447 & 9.899 & 8.880 & 8.172 \\
2\_2\_3 & 9.070 & 10.114 & 9.298 & 11.977 & 7.769 & 10.307 & 9.330 & 8.419 \\
2\_2\_4 & 12.097 & 9.784 & 9.146 & 12.242 & 7.909 & 11.541 & 11.187 & 8.727 \\
2\_2\_5 & 11.684 & 13.602 & 11.329 & 8.226 & 9.902 & 10.300 & 8.976 & 8.887 \\
\hline
\end{longtable}

\begin{longtable}{%
P{0.10\textwidth}%
P{0.12\textwidth}%
P{0.10\textwidth}%
P{0.10\textwidth}%
P{0.10\textwidth}%
P{0.10\textwidth}%
P{0.10\textwidth}%
P{0.10\textwidth}%
P{0.10\textwidth}%
}
\caption{\textbf{Per-instance runtime for the random hybrid-constraint benchmark.}
 All runtimes are reported in seconds. 
 The label “to” indicates that the solver exceeded the allotted time limit (1000 seconds) and did not return a solution.}
 \label{tab:random}\\
\hline
Instances & \textsc{FourierSMT} & \textsc{Yices2} & \textsc{SMTS} & \textsc{MathSAT5} & \textsc{Z3} & \textsc{Z3PP} & \textsc{SMT-RAT} & \textsc{CVC5} \\
\hline
100\_0 & 10.320 & 0.667 & 2.056 & 1.206 & 51.959 & 127.370 & 1.774 & 2.860 \\
100\_1 & 10.520 & 0.683 & 2.260 & 1.190 & 51.131 & 98.331 & 1.554 & 2.693 \\
100\_2 & 11.720 & 0.673 & 2.402 & 1.201 & 45.104 & 60.249 & 1.608 & 2.737 \\
100\_3 & 13.660 & 0.663 & 2.342 & 1.155 & 41.762 & 95.694 & 1.631 & 2.874 \\
100\_4 & 14.290 & 0.698 & 2.386 & 1.232 & 58.014 & 99.942 & 1.596 & 2.981 \\
100\_5 & 14.500 & 0.669 & 19.931 & 1.200 & 46.701 & 52.544 & 1.583 & 2.711 \\
\hline
\end{longtable}

\begin{longtable}{%
P{0.10\textwidth}%
P{0.12\textwidth}%
P{0.10\textwidth}%
P{0.10\textwidth}%
P{0.10\textwidth}%
P{0.10\textwidth}%
P{0.10\textwidth}%
P{0.10\textwidth}%
P{0.10\textwidth}%
}
\caption{\textbf{Per-instance runtime for the continuous time scheduling benchmark.}
 All runtimes are reported in seconds. 
 The label “to” indicates that the solver exceeded the allotted time limit (1000 seconds) and did not return a solution.}
 \label{tab:scheduling}\\
\hline
Instances & \textsc{FourierSMT} & \textsc{Yices2} & \textsc{SMTS} & \textsc{MathSAT5} & \textsc{Z3} & \textsc{Z3PP} & \textsc{SMT-RAT} & \textsc{CVC5} \\
\hline
16\_2\_0 & 27.040 & 0.230 & 1.172 & 0.299 & 0.862 & 0.331 & 9.956 & 1.179 \\
16\_2\_1 & 16.922 & 0.062 & 0.195 & 0.163 & 0.083 & 0.163 & 9.952 & 0.332 \\
16\_2\_2 & 12.897 & 0.030 & 0.358 & 0.252 & 0.092 & 0.154 & 7.262 & 0.288 \\
16\_2\_3 & 12.357 & 0.044 & 0.212 & 0.324 & 0.096 & 0.159 & 5.562 & 0.267 \\
16\_2\_4 & 11.105 & 0.032 & 0.199 & 0.261 & 0.086 & 0.145 & 6.632 & 0.252 \\
16\_2\_5 & 9.473 & 0.033 & 0.194 & 0.323 & 0.086 & 0.157 & 7.903 & 0.266 \\
\hline
\end{longtable}

\begin{longtable}{%
P{0.10\textwidth}%
P{0.12\textwidth}%
P{0.10\textwidth}%
P{0.10\textwidth}%
P{0.10\textwidth}%
P{0.10\textwidth}%
P{0.10\textwidth}%
P{0.10\textwidth}%
P{0.10\textwidth}%
}
\caption{\textbf{Per-instance runtime for the 3D placement benchmark.}
 All runtimes are reported in seconds. 
 The label “to” indicates that the solver exceeded the allotted time limit (1000 seconds) and did not return a solution.}
 \label{tab:placement}\\
\hline
Instances & \textsc{FourierSMT} & \textsc{Yices2} & \textsc{SMTS} & \textsc{MathSAT5} & \textsc{Z3} & \textsc{Z3PP} & \textsc{SMT-RAT} & \textsc{CVC5} \\
\hline
2\_2\_0 & 8.357 & 0.005 & 0.120 & 0.020 & 0.034 & 0.050 & 0.213 & 0.023 \\
2\_2\_1 & 8.462 & 0.005 & 0.122 & 0.021 & 0.036 & 0.051 & 0.214 & 0.024 \\
2\_2\_2 & 8.464 & 0.006 & 0.132 & 0.021 & 0.036 & 0.052 & 0.228 & 0.024 \\
2\_2\_3 & 8.573 & 0.006 & 0.136 & 0.021 & 0.036 & 0.052 & 0.373 & 0.025 \\
2\_2\_4 & 8.851 & 0.006 & 0.136 & 0.021 & 0.036 & 0.052 & 0.373 & 0.025 \\
2\_2\_5 & 8.908 & 0.006 & 0.138 & 0.021 & 0.036 & 0.053 & 0.379 & 0.025 \\
\hline
\end{longtable}

\clearpage

\begin{figure}[h]
    \centering
    \includegraphics[width=\linewidth]{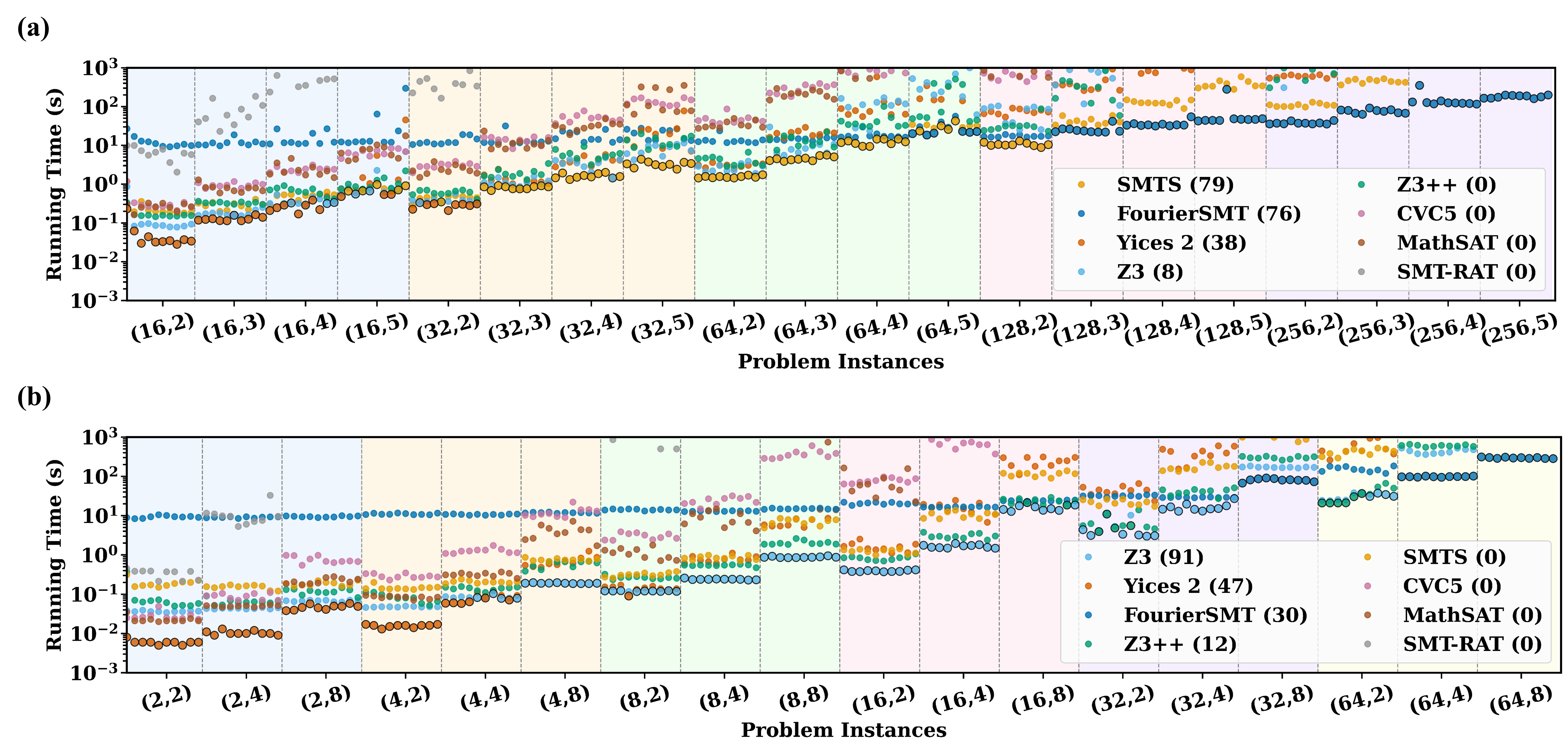}
    \caption{\textbf{Detailed solver performance on benchmark categories.}
        \textbf{(a)} Results on the scheduling problems. 
        {\textsc{Z3++}, \textsc{CVC5}, \textsc{MathSAT5}, \textsc{SMT-RAT} wins non of the instances in this category.}
        \textsc{Yices2} performs best on the small-scale instances ($n_w = 16$), while \textsc{SMTS} attains the best results on the medium-scale cases ($n_w \in \{32, 64\}$). 
        As the instance size increases further, \textsc{FourierSMT} can handle cases that cause timeouts among competing solvers.
        \textbf{(b)} Results on the placement problems. 
        {\textsc{SMTS}, \textsc{CVC5}, \textsc{MathSAT5}, \textsc{SMT-RAT} wins non of the instances in this category.}
        \textsc{Yices2} still performs best on the small-scale instances ($n_m = 2$), while \textsc{Z3} and \textsc{Z3++} attains the best results on the medium-scale cases ($n_w \in \{8, 16, 32\}$). 
        {The landscape of solvers exhibits pronounced variability: different solvers excel in different problem scales, and no classical solver maintains uniformly strong performance.
        \textsc{FourierSMT} benefits from GPU acceleration, resulting in relatively similar solving times across instance sizes because the fixed parallelization overhead of the GPU amortizes most of the computational cost.}
        For the largest category, \textsc{FourierSMT} remains scalable and continues to solve instances that all other solvers time out on.}
    \label{fig:by-instance}
\end{figure}

\newpage
\section{Gradient Computation Time}\label{note:gradient}
To quantify the acceleration attributable to GPU execution, we measure the gradient computation time of \textsc{FourierSMT} when executed on either CPU or GPU.
Our high-performance cluster node has dual AMD EPYC 9654 (with at most 384 threads), but configure the following flags: \texttt{XLA\_FLAGS=--xla\_cpu\_multi\_thread\_eigen=false intra\_op\_parallelism\_threads=64}, \texttt{OMP\_NUM\_THREADS=64}, and \texttt{MKL\_NUM\_THREADS=64} to enforce that \textsc{FourierSMT} uses at most 64 CPU threads.
Similarly, the node contains 8 NVIDIA L40S GPUs, but we constrain \textsc{FourierSMT} to a single GPU by setting \texttt{CUDA\_VISIBLE\_DEVICES=0}.

We measure the gradient computation time using the instances in the random hybrid constraint benchmark.
For each instance, we evaluate the gradient at 10,000 distinct initial points.
These points are partitioned into 100 groups, where each group of 100 points uses the same annealing parameter.
We test 100 annealing parameters corresponding to $1/\sigma\in\{0.00, 0.01, 0.99\}$.
The results of this experiment are reported in Fig.~\ref{fig:random}b of the main text.
In this section, we report the per-instance gradient computation time in Table~\ref{tab:gradient-time}.

\begin{longtable}{%
P{0.10\textwidth}%
P{0.12\textwidth}%
P{0.12\textwidth}%
}
\caption{\textbf{Per-instance average gradient runtime of \textsc{FourierSMT} for the random hybrid-constraint benchmark.}
 All runtimes are reported in milliseconds. 
 The label “to” indicates that the solver exceeded the allotted time limit (1000 seconds) and did not return a solution.}
 \label{tab:gradient-time}\\
\hline
Instances & CPU & GPU \\
\hline
100\_0 & 4.841 & 12.890 \\
100\_1 & 4.993 & 12.848 \\
100\_2 & 4.975 & 12.126 \\
100\_3 & 5.300 & 12.370 \\
100\_4 & 4.930 & 12.902 \\
100\_5 & 4.886 & 12.942 \\
100\_6 & 5.251 & 12.888 \\
100\_7 & 4.861 & 12.973 \\
100\_8 & 4.911 & 12.756 \\
100\_9 & 4.884 & 12.323 \\
200\_0 & 23.869 & 13.734 \\
200\_1 & 21.147 & 13.594 \\
200\_2 & 21.687 & 13.708 \\
200\_3 & 22.092 & 13.645 \\
200\_4 & 20.119 & 13.596 \\
200\_5 & 20.233 & 13.531 \\
200\_6 & 22.756 & 13.481 \\
200\_7 & 19.925 & 13.689 \\
200\_8 & 24.230 & 13.694 \\
200\_9 & 21.735 & 13.693 \\
300\_0 & 35.270 & 12.828 \\
300\_1 & 34.592 & 12.738 \\
300\_2 & 33.986 & 12.485 \\
300\_3 & 34.086 & 12.560 \\
300\_4 & 34.517 & 12.774 \\
300\_5 & 32.877 & 12.885 \\
300\_6 & 34.472 & 12.621 \\
300\_7 & 33.741 & 12.883 \\
300\_8 & 33.282 & 12.712 \\
300\_9 & 34.676 & 13.236 \\
400\_0 & 41.611 & 12.906 \\
400\_1 & 41.828 & 12.971 \\
400\_2 & 42.041 & 12.790 \\
400\_3 & 43.206 & 12.829 \\
400\_4 & 44.231 & 12.649 \\
400\_5 & 42.228 & 12.518 \\
400\_6 & 42.131 & 12.778 \\
400\_7 & 46.592 & 12.903 \\
400\_8 & 42.831 & 12.731 \\
400\_9 & 44.681 & 12.893 \\
500\_0 & 48.972 & 12.831 \\
500\_1 & 50.528 & 12.483 \\
500\_2 & 48.619 & 12.816 \\
500\_3 & 48.682 & 12.483 \\
500\_4 & 47.275 & 12.462 \\
500\_5 & 47.327 & 12.348 \\
500\_6 & 47.238 & 12.817 \\
500\_7 & 47.029 & 12.674 \\
500\_8 & 51.214 & 12.514 \\
500\_9 & 47.077 & 12.605 \\
600\_0 & 59.570 & 12.863 \\
600\_1 & 55.350 & 12.784 \\
600\_2 & 52.882 & 12.745 \\
600\_3 & 53.691 & 12.688 \\
600\_4 & 58.432 & 12.558 \\
600\_5 & 55.391 & 12.882 \\
600\_6 & 54.917 & 12.973 \\
600\_7 & 55.205 & 13.011 \\
600\_8 & 58.445 & 12.659 \\
600\_9 & 58.832 & 12.750 \\
700\_0 & 59.736 & 12.797 \\
700\_1 & 65.506 & 12.836 \\
700\_2 & 66.118 & 13.046 \\
700\_3 & 60.359 & 12.455 \\
700\_4 & 63.349 & 12.727 \\
700\_5 & 60.291 & 12.665 \\
700\_6 & 60.050 & 12.934 \\
700\_7 & 60.203 & 12.671 \\
700\_8 & 68.391 & 12.668 \\
700\_9 & 60.000 & 12.822 \\
800\_0 & 71.691 & 12.832 \\
800\_1 & 77.569 & 12.737 \\
800\_2 & 66.064 & 12.634 \\
800\_3 & 65.728 & 12.997 \\
800\_4 & 65.247 & 12.726 \\
800\_5 & 70.436 & 12.787 \\
800\_6 & 71.175 & 12.987 \\
800\_7 & 67.009 & 12.992 \\
800\_8 & 67.247 & 12.815 \\
800\_9 & 68.656 & 12.895 \\
900\_0 & 77.775 & 12.671 \\
900\_1 & 80.032 & 12.928 \\
900\_2 & 75.771 & 12.982 \\
900\_3 & 82.115 & 13.027 \\
900\_4 & 77.760 & 12.705 \\
900\_5 & 76.144 & 12.758 \\
900\_6 & 76.445 & 12.863 \\
900\_7 & 75.271 & 12.708 \\
900\_8 & 76.232 & 13.017 \\
900\_9 & 76.164 & 13.072 \\
1000\_0 & 90.081 & 12.284 \\
1000\_1 & 83.093 & 12.864 \\
1000\_2 & 84.637 & 12.743 \\
1000\_3 & 83.680 & 12.876 \\
1000\_4 & 82.626 & 12.527 \\
1000\_5 & 82.269 & 12.627 \\
1000\_6 & 82.549 & 12.778 \\
1000\_7 & 86.008 & 12.697 \\
1000\_8 & 82.360 & 12.418 \\
1000\_9 & 86.246 & 12.711 \\
\hline
\end{longtable}


\end{document}